\documentclass[hidelinks]{article}
\usepackage[letterpaper, portrait, top=1in, left=1in, right=1in, bottom=1in ]{geometry}

\usepackage{centernot}

\usepackage[utf8]{inputenc} 
\usepackage[T1]{fontenc}    
\usepackage{hyperref}       
\usepackage{url}            
\usepackage{booktabs}       
\usepackage{amsfonts,amsmath,amsthm}       
\usepackage{nicefrac}       
\usepackage{microtype}      
\usepackage{authblk}

\usepackage{color}
\usepackage{graphicx}
\usepackage{caption}
\usepackage{amsmath}
\usepackage{amsthm}
\usepackage{enumitem}
\usepackage{subfig}
\usepackage{appendix}

\newtheorem{definition}{Definition}

\newtheorem{lemma}{Lemma}
\newtheorem{theorem}{Theorem}
\newtheorem{corollary}{Corollary}

\usepackage[boxed]{algorithm}
\usepackage{graphicx}
\usepackage{capt-of}
\usepackage{algorithmicx}
\usepackage{algpseudocode}

\newcommand{\x}{\mathbf{x}}
\newcommand{\y}{\mathbf{y}}
\newcommand{\z}{\mathbf{z}}
\newcommand{\X}{\mathbf{X}}
\newcommand{\Y}{\mathbf{Y}}
\newcommand{\Z}{\mathbf{Z}}

\newcommand{\EE}{\mathbb{E}}
\newcommand{\ci}{\mathrel{\perp\mspace{-10mu}\perp}}

\title{Mimic and Classify : A \textit{meta}-algorithm for Conditional Independence Testing}

\author[3]{Rajat Sen\thanks{rajat.sen@utexas.edu}}
\author[2]{Karthikeyan Shanmugam\thanks{karthikeyan.shanmugam2@ibm.com}}
\author[1]{Himanshu Asnani\thanks{asnani@uw.edu}}
\author[1]{Arman Rahimzamani\thanks{armanrz@uw.edu}}
\author[1]{Sreeram Kannan\thanks{ksreeram@uw.edu}}
\affil[1]{Department of Electrical Engineering, University of Washington, Seattle, WA}
\affil[2]{IBM Research}
\affil[3]{Department of Electrical and Computer Engineering, University of Texas at Austin, TX}

\date{}

\begin{document}

\maketitle

\begin{abstract}
Given independent samples generated  from the joint distribution $p(\mathbf{x},\mathbf{y},\mathbf{z})$, we study the problem of Conditional Independence (CI-Testing), i.e., whether the joint equals the CI distribution $p^{CI}(\mathbf{x},\mathbf{y},\mathbf{z})= p(\mathbf{z}) p(\mathbf{y}|\mathbf{z})p(\mathbf{x}|\mathbf{z})$  or not. We cast this problem under the purview of the proposed, provable \textit{meta}-algorithm, \textbf{"Mimic and Classify"}, which is realized in two-steps: (a) \textit{Mimic} the CI distribution close enough to recover the support, and (b) \textit{Classify} to distinguish the joint and the CI distribution. Thus, as long as we have a good generative model and a good classifier, we potentially have a sound CI Tester. With this modular paradigm, CI Testing becomes amiable to be handled by state-of-the-art, both generative and classification methods from the modern advances in Deep Learning, which in general can handle issues related to \textit{curse of dimensionality} and operation in \textit{small sample regime}. We show intensive numerical experiments on synthetic and real datasets where new \textit{mimic} methods such conditional GANs, Regression with Neural Nets, outperform the current best CI Testing performance in the literature. Our theoretical results provide analysis on the estimation of null distribution as well as allow for general measures, i.e., when either some of the random variables are discrete and some are continuous or when one or more of them are discrete-continuous \textit{mixtures}.
\end{abstract}

\section{Introduction}

\label{sec:motivation}
Conditional Independence Testing is central in problems on causal discovery and statistical inference in several dynamical systems, such as gene regulatory networks, finance networks, edge testing in Bayesian Networks etc. (cf. \cite{pearl_book}, \cite{spirtes_book}, \cite{koller_book}, \cite{peters_book}). The problem of \textit{Conditional Independent Testing} (\textbf{CI-Testing} in short) is as follows: \textit{Given $n$ i.i.d. samples $(\x_i,\y_i,\z_i)_{i=1}^n$ from the joint probability distribution $p(\mathbf{x},\mathbf{y},\mathbf{z})$, distinguish between the two hypotheses :    $\mathcal{H}_0:  p(\mathbf{x},\mathbf{y},\mathbf{z}) = p^{CI}(\mathbf{x},\mathbf{y},\mathbf{z}) = p(\mathbf{z}) p(\mathbf{y}|\mathbf{z}) p(\mathbf{x}|\mathbf{z})$ and $ \mathcal{H}_1: p(\mathbf{x},\mathbf{y},\mathbf{z})= p(\mathbf{z}) p(\mathbf{y}|\mathbf{z})p(\mathbf{x}|\mathbf{y},\mathbf{z})$}. Our focus here is to design an overall methodology of non-parametric \textbf{CI-Testing} addressing all of them beyond the state of the art:

\textbf{(C1)}  \textbf{Estimating Null Distribution:} Benchmarking the performance of any CI test by providing exact or approximate analytical estimates on the null distribution.

\textbf{(C2)} \textbf{Small Sample Regime:} Designing the test with good performance in small sample regime; for e.g., density-estimation-based methods would routinely need a large number of samples.

\textbf{(C3)}  \textbf{Curse of Dimensionality:} Addressing the issue of large conditioning set; for e.g., to causally infer an edge over a large graph, practically the whole graph is the conditioning set. 

\textbf{(C4)} \textbf{General Measures:} Handling the case of \textit{mixture} of continuous and discrete components, that is beyond densities to general measures, in defining the \textbf{CI-Testing} problem above.

\subsection{Prior Art}
Much of the prior work focussed on handling \textbf{(C1)} by explicitly estimating conditional densities or functionals thereof, and conditional independence is observed via calculating an appropriate test statistics (\cite{su_white_2007}, \cite{su_white_2008}) or via discretizing the conditioning set (\cite{margaritis_2005}, \cite{huang_2010}). These approaches naturally fail on other concerns, \textbf{(C2)} and \textbf{(C3)}. Several non-parametric approaches based on \textit{kernel} methods and equivalent characterizations in terms of cross-covariance operator on the corresponding Reproducing  Kernel Hilbert Spaces (RKHSs)(\cite{fukumizu_2004}, \cite{gretton_2008}, \cite{fukumizu_2008}), in general fail to satisfy \textbf{(C2)} as they suffer from high computational complexity owing to the invertibility issues of large matrices as well as the non-robustness associated with the adjustment of the bandwidth parameters. 

Building on the idea of kernel trick and local permutation, \cite{kcit_2011} proposed Kernel Conditional Independence Test (KCIT) which harnesses partial association of regression functions (\cite{daudin_1980}). Approximate and faster versions of KCIT were proposed as Randomized Conditional Independence Test (RCIT) and variants (\cite{rcit_2017}). Conditional Distance Independence Test (CDIT) was proposed in \cite{cdit_2015} which uses conditional correlation of conditional characteristic functions. Alternatively by using the proven efficacy of \textit{classification methods} (\cite{boucheron_2005}, \cite{xgboost_2016}, \cite{imagenet_2012}),  Kernel Conditional Independence Permutation Test (KCIPT) (\cite{pcit_2014}) and Classifier Conditional Independence Test (CCIT)(\cite{ccit_2017}) concatenate local permutations and nearest-neighbor bootstrap, respectively, with a two-sample test to distinguish between the two hypotheses. However both the local permutation and nearest neighbor methods can potentially suffer from searching in the space $\Z$ when it is quite large besides needing more samples. Conditional Mutual Information Test (CMIT) (\cite{cmit_2018}) based on the estimators of Conditional Mutual Information (\cite{kozachenko1987sample}, \cite{singh2003nearest}, \cite{kraskov2004estimating}, \cite{gao2017mixture}, \cite{gao2016conditional}, \cite{pramod_demystify_2018}) lacks any sound theoretical framework \textbf{(C1)}. Note that none of the above methods, \textit{kernel} or otherwise, have been designed to work on general measures \textbf{(C4)}.

Faced with these limitations we consider that the more successful CI testers of the lot employ a two-step approach: (a) Try to get as close to the CI distribution, and then (b) Use the power of supervised learning by reducing the CI Testing problem to that of a binary classification. In this way they try to mitigate the concerns \textbf{(C1-C3)}. No doubt that classification methodology helps leverage already mature technology to handle the concern of theoretical guarantees on null distribution as well as performing well in the high-dimensional regime. However in the high dimensional regime, it will be too much to expect from nearest-neighbor and local permutation methods. \textit{Could there be any other technology as an alternative in Step (a)?} For instance can Step (a) incorporate Generative Adversarial Methods (GANs) (\cite{gan_2014}, \cite{cgan_2014})? However GANs also face the bottleneck of not being able to approximating the density closely, for instance with respect to multi-modal distributions (\cite{arora_gan_1_2018}, \cite{arora_gan_2_2018}). Nonetheless, motivated by our experiments with GANs and other deep learning methods such as Regression with Neural Nets in Section \ref{sec:sims}, which show improved performance, we were forced to ask the following bold question: \textbf{\textit{Is it sufficient to only approximate the CI distribution in some loose sense instead of approximating it closely?}} This investigation helped us develop the main idea behind this paper.

\subsection{Main Contributions}

We answer the above question in the affirmative, that is, we are good as long as we \textbf{mimic} the CI distribution reasonably closely (cf. the main theorem in the paper, Theorem \ref{thm:mainthm1} in Section \ref{sec:theory}). This is suggestive of a new modular paradigm and philosophy of CI Testing which we introduce in our work - \textbf{"Mimic and Classify"}, which is a essentially a two-step approach: 

Step 1 \textbf{Mimic:} Use any known "good" off-the-shelf generative methods to mimic the CI distribution.

Step 2 \textbf{Classify:} Perform a classification test distinguishing the joint and the mimicked distribution. 

Hence as long as we have a good generator in the sense of Theorem \ref{thm:mainthm1}, i.e. they recover the support of the CI distribution, and a classifier, we can potentially have a good CI Tester. This modular approach not only generalizes the existing methods,  but also provides a general paradigm for a wide class of methods, including those from the latest advances in Deep Learning (\cite{dl_book_2016}) to be used for not only classification step but also for generative (mimicking) step. We thus have a modular methodology which potentially addresses all the concerns, \textbf{(C1-C4)} noted above. We can therefore summarize the main contributions and the paper organization as follows:

\textbf{ A Modular Approach:} With mimic and classify approach, we obtain a general methodology for creating efficient CI Testers which can use methods from Deep Learning to mitigate the problem of \textit{curse of dimensionality} \textbf{(C3)}, such as in GANs, as well as circumvent the concern of \textit{small sample regime} \textbf{(C2)}, such as in Regression methods. Section \ref{sec:approach} presents the main Algorithm \ref{algm:CI-Test}, Sections \ref{sec:bayesoptimal} and \ref{sec:analysis} presents the theoretical analysis.

 \textbf{Discrete-Continuous Mixtures:}  The theoretical results have been shown to exist even when the samples are generated from a \textit{mixture} or general measure \textbf{(C4)}. This is dealt in Section \ref{sec:general-measures} where the general Theorem \ref{thm:mainthm2} is proved for general measures.
 
 \textbf{Generalization Bounds:} Section \ref{sec:pvalue} outlines the results of the generalization risk of our methodology, thus giving us theoretical estimates on the null distribution  \textbf{(C1)}.
 
\textbf{Empirical Evaluation:} Several candidates for mimicking the CI distribution such as Conditional GANs, Regression with Neural Nets are studies in Section \ref{sec:mimic} and are run  on both synthetic and real datasets in Section \ref{sec:sims}. Our results show they outperform the current state-of-the-art.

\section{Our Approach: Mimic and Classify}
\label{sec:approach}
We have the \textbf{CI-Testing} problem as defined in Section \ref{sec:motivation} with $p$ data samples $(\mathbf{x}_i,\mathbf{y}_i,\mathbf{z}_i) \in \mathbb{R}^{1 \times n_x} \times \mathbb{R} ^{1 \times n_y}\times \mathbb{R}^{1 \times n_z}$ ($1 \leq i \leq p$) drawn i.i.d from a joint distribution whose density function is given by $p(\mathbf{x},\mathbf{y},\mathbf{z})$ (we drop the subscripts in the notation of density for simplicity, hence $p(\x,\y,\z)$ stands for $p_{\X\Y\Z}(\x,\y,\z)$ and likewise for other marginal and conditional densities). Let us denote the data set by $D$. In other words, we will assume that the joint measure of the random variables $(\mathbf{X},\mathbf{Y},\mathbf{Z})$ is absolutely continuous with respect to the Lebesgue Measure on $\mathbb{R}^{1 \times n_x} \times \mathbb{R} ^{1 \times n_y}\times \mathbb{R}^{1 \times n_z}$, thus permitting a density function. Our results apply more generally even for general measures (cf. Section \ref{sec:general-measures}). However, for now, we will state all the results under this assumption to make the exposition clear. Formal algorithm description is in Algorithm \ref{algm:CI-Test}, and it is schematically depicted in Figure~\ref{fig:CI-Test_schematic}. We first describe the important steps in the meta-algorithm informally as follows:

a) \textbf{Mimic the CI Distribution:} We take the input dataset $D$ and divide it into two data sets $D_1$ and $D_2$. We create another data set $D'$ using points in $D_2$ that "mimics" the CI distribution by approximating it with $p(\mathbf{z}) q(\mathbf{y}|\mathbf{z}) p(\mathbf{x}|\mathbf{z})$ for some conditional density  $q(\cdot | \mathbf{z})$. We use the function $\mathrm{MIMIC} (\cdot)$ in Algorithm \ref{algm:CI-Test} to obtain $D'$. Here any particular $\mathrm{MIMIC} (\cdot)$ function is unspecified, except in the experiments, for it can encompass a broad spectrum of methods to attain its goal. Hence we derive theoretical conditions based on the density function $q(\mathbf{y}|\mathbf{z})$. 

b) \textbf{Label:} Label $D'$ and $D_1$ using distinct labels (say $0$ and $1$ respectively). Let $\tilde{D}=D' \bigcup D_1$.

c) \textbf{Two Binary Classifier Tests:} Drop the variable $x$ from $\tilde{D}$ and form $\tilde{D}^{-\mathbf{x}}$. Let  ${\cal C}$ be a class of binary classifiers. We train separately two binary classifiers, $f_1\in \cal C$ and $f_2\in \cal C$:

   $f_1:(\mathbf{y},\mathbf{z}) \rightarrow \{0,1\}$, i.e., $f_1$ is trained on $\tilde{D}^{-\mathbf{x}}$.
   
  $f_2:(\mathbf{x},\mathbf{y},\mathbf{z}) \rightarrow \{0,1\}$, i.e., $f_2$ is trained on $\tilde{D}$ using all its features.

d) \textbf{Check Separation in Errors:} If the classification errors from both these cases are well-separated, we reject the null hypothesis $\mathcal{H}_0$.

\textbf{Note:} With some abuse of notation, we use the same classifier class ${\cal C}$ for the two cases even when the feature spaces are different (one is $\mathbb{R}^{1 \times n_x} \times \mathbb{R} ^{1 \times n_y}\times \mathbb{R}^{1 \times n_z}$while the other is $\mathbb{R} ^{1 \times n_y}\times \mathbb{R}^{1 \times n_z}$). In this context, it is understood that loosely speaking classifiers for both the tests come from the same family (e.g.: Logistic Regression, Gradient boosted Trees, Neural Networks etc.) characterized by a specific training algorithm.

\begin{algorithm}[htbp]
\caption{CI-Testing by Mimicking CI distribution}
\label{algm:CI-Test}.
\begin{algorithmic}[1]
\Require{\textbf{Dataset} $D \sim p(\mathbf{x},\mathbf{y},\mathbf{z})$, \textbf{Training Algorithm} ${\cal A}(D_t,D_v,D_s)$ that trains a binary classifier given labelled train, validation and test datasets $D_t,D_v$ and $D_s$ and outputs a classifier from the class ${\cal C}$, \textbf{Threshold} $\tau$.}
\Ensure{Hypothesis: $\mathcal{H}_0$ or $\mathcal{H}_1$.}
\State Divide $D$ into two data sets $D_1$ and $D_2$.
\State $D'=\mathrm{MIMIC}(D_2)$. Let $D' \sim p(\mathbf{z})q(\mathbf{y}|\mathbf{z}) p(\mathbf{x}|\mathbf{z})$.
\State Let $D'$ be labeled with $0$ and let $D_1$ be labeled with $1$. Let $\tilde{D}=D' \bigcup D_1$. Split $\tilde{D}$ into train, validation and test datasets denoted by $\tilde{D}_t,\tilde{D}_v$ and $\tilde{D}_s$ respectively. 
\State Drop the variable $\mathbf{x}$ from $\tilde{D}$ and form the data set $\tilde{D}^{-\mathbf{x}}$. Let $\tilde{D}^{-\mathbf{x}}_t,\tilde{D}^{-\mathbf{x}}_v$ and $\tilde{D}^{-\mathbf{x}}_s$ denote the train, validation and test parts.
\State \label{f1_line} Train a classifier $f_1:(\mathbf{y},\mathbf{z}) \rightarrow \{0,1\} \in {\cal C}$ using the algorithm ${\cal A}(\tilde{D}^{-\mathbf{x}}_t,\tilde{D}^{-\mathbf{x}}_v,\tilde{D}^{-\mathbf{x}}_s)$. Let the classification error be $e(f_1,\tilde{D}^{-\mathbf{x}}_s)$.
\State \label{f2_line} Train another classifier $f_2:(\mathbf{x},\mathbf{y},\mathbf{z})\rightarrow \{0,1\} \in {\cal C}$ using the algorithm ${\cal A}(\tilde{D}_t,\tilde{D}_v,\tilde{D}_s)$. Let the classification error on the test set be $e(f_2,\tilde{D}_s)$.
\If {$\lvert e(f_2,\tilde{D}_s) - e(f_1,\tilde{D}^{-\mathbf{x}}_s) \rvert > \tau$}
    \State Return $\mathcal{H}_1$.
\Else
     \State Return $\mathcal{H}_0$.
\EndIf     

\end{algorithmic}
\end{algorithm}

\begin{figure}[t]
\centering
\includegraphics[width=.75\textwidth,trim={0cm 4.5cm 0cm 3.38cm},clip]{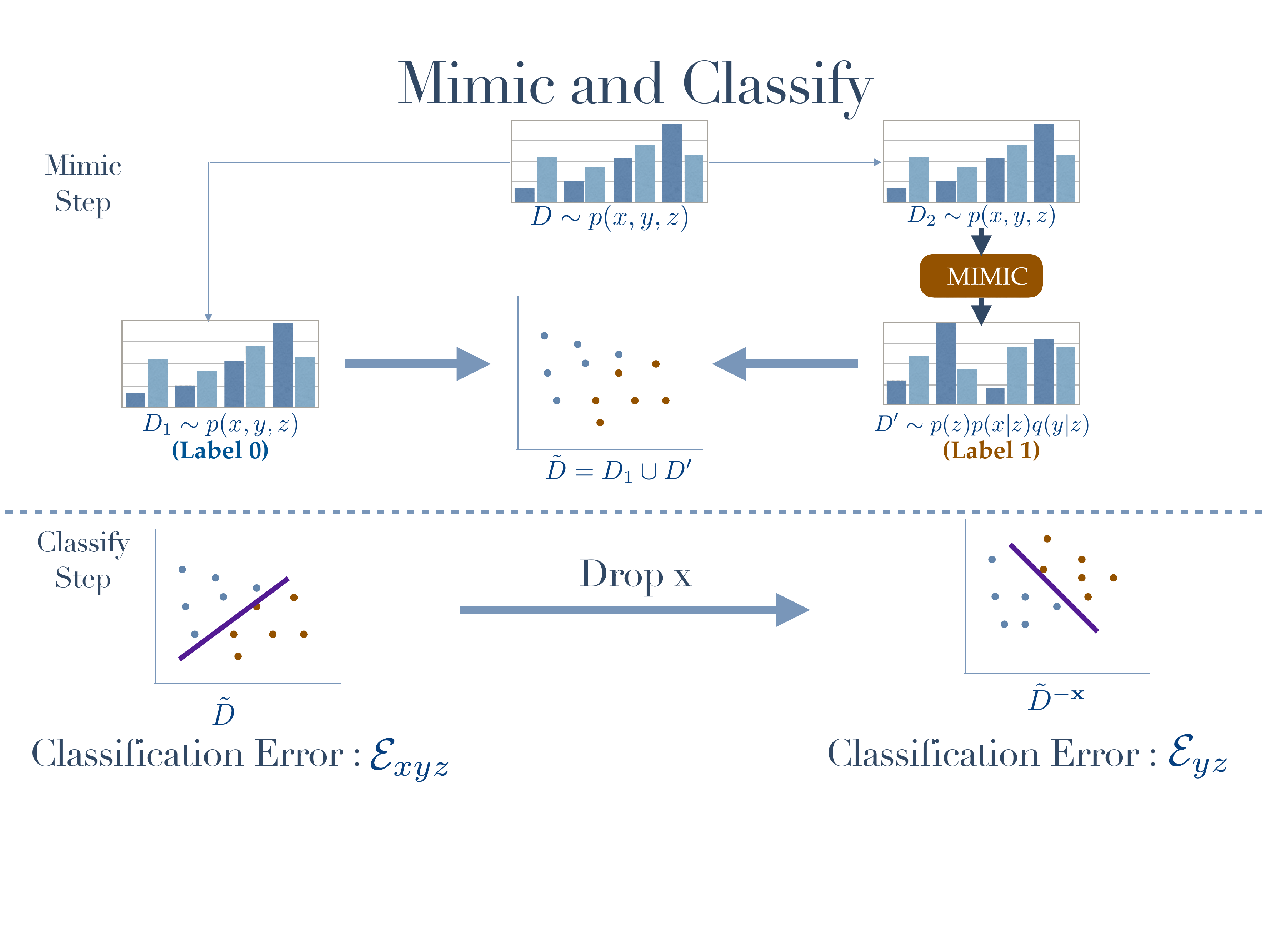}
\caption{A schematic of Algorithm~\ref{algm:CI-Test}. The upper section shows the \textbf{mimicking} procedure (Steps 1-4) and the lower section shows the \textbf{classification} procedure (Steps 5-11). The algorithm returns $\mathcal{H}_1$ if $| \mathcal{E}_{xyz} - \mathcal{E}_{yz} | < \tau$ and $\mathcal{H}_0$ otherwise.}
\label{fig:CI-Test_schematic}
\end{figure}

\section{Theoretical Results with the Bayes Optimal Classifiers}
\label{sec:theory}
 In this section, we analyze the meta algorithm in Algorithm \ref{algm:CI-Test} under the assumption that the algorithm ${\cal A}$ finds the Bayes optimal classifier. Let $f^{*}_1$ and $f^{*}_2$ be the Bayes optimal classifiers returned by ${\cal A}$ during the respective calls in Algorithm \ref{algm:CI-Test}. Then we show the following theorem:
 
 \begin{theorem}\label{thm:mainthm1}
   As long as the density function $q(\mathbf{y}|\mathbf{z})>0$ whenever $p(\mathbf{y},\mathbf{z})>0$, $\lvert \mathbb{E}_D[e(f^{*}_1,\tilde{D}_s)]- \mathbb{E}_D[e(f^{*}_2,\tilde{D}^{-\mathbf{x}}_s)]  \rvert =0$ if and only if hypothesis $\mathcal{H}_0$ is true.
   \end{theorem}
 
 \textbf{Remark:} This means that it is enough for $q(\mathbf{y}|\mathbf{z})$ to match the support of $p(\mathbf{y}|\mathbf{z})$. It is not needed for $q(\cdot)$ to closely approximate $p(\cdot)$ in some distance. The exact and stronger characterization is given in Theorem \ref{thm:maintv}. This basically provides us an array of options for the $\mathrm{MIMIC}()$ function in Algorithm \ref{algm:CI-Test}. 
 
 
 \subsection{Bayes Optimal Classifiers and Total Variation Distance}
 \label{sec:bayesoptimal}
 We now review some theoretical preliminaries regarding the total variation distance and its relationship with Bayes optimal classifiers.  Let $p(\mathbf{w})$ and $q(\mathbf{w})$ be two different density functions on $\mathbf{R}^{1 \times s}$. Form a data set $D$ containing $p$ i.i.d. samples of the form $(\mathbf{w},\ell)$ where $\ell$ is a Bernoulli random variable with bias probability $0.5$. If $\ell=1$, then $\mathbf{w} \sim p(\cdot)$ and when $\ell=0$, then $\mathbf{w} \sim q(\cdot)$. Let us consider the space of classifiers $f:\mathbf{w} \rightarrow \{0,1\}$. Let $\mathbb{E}_{D}[e(f,D)]$ be the expected error for any classifier $f(\cdot)$. We can characterize the optimal classifier \cite{tong2000restricted} (proof is omitted as it is well-known) as follows:
 
 \begin{lemma}\label{lem:bayesopt}
 (folklore) The Bayes optimal classifier denoted by $f^{*}:\mathbf{w} \rightarrow \{0,1\}$ is: $\ell=1$ if $\mathbf{w}\in A$ and $\ell=0$ otherwise, where $A = \{\mathbf{w} : p(\mathbf{w}) > q (\mathbf{w})\}$. Further,
     $ \min \limits_{f} \mathbb{E}_D[e(f,D)] =  \mathbb{E}_D[e(f^{*},D)] = \frac{1}{2}\int \min (p(\mathbf{w}),q(\mathbf{w})) d\mathbf{w} $
 \end{lemma}

 We recall the following (equivalent) definitions of the total variation distance.
 
 \begin{definition}(Total Variation Distance)
 \label{def:tvar}
  Let $P$ and $Q$ be measures on $\mathbb{R}^{1 \times s}$ that are absolutely continuous with respect to the Lebesgue measure equipped with density functions $p(\mathbf{w}), ~\mathbf{w} \in \mathbb{R}^{1 \times s}$ and $q(\mathbf{w}), ~\mathbf{w} \in \mathbb{R}^{1 \times s}$ respectively. Let ${\cal F}$ denote the underlying sigma algebra. The total variation distance between measures $P$ and $Q$ denoted by $D_{\mathrm{TV}} (P,Q)$ is given by: 
   \begin{align}
      D_{\mathrm{TV}}(P,Q) & =  \sup_{A \in {\cal F}} \lvert P(A) - Q(A) \rvert  \label{tv:event}\\
      D_{\mathrm{TV}}(P,Q) & = \frac{1}{2} \int \lvert p(\mathbf{w}) -q(\mathbf{w}) \rvert d\mathbf{w} \label{tv:int} 
   \end{align}
 Total variation distance is also the optimal transportation cost over all coupling between measures $P$ and $Q$ when the cost function is $c(\mathbf{w},\mathbf{\tilde{w}}) = \mathbf{1}_{\{\mathbf{w} \neq \mathbf{\tilde{w}}\}}$, where $\mathbf{1}_{\{\cdot\}}$ is the indicator function. Let $\Pi(P,Q)$ be the set of all joint distributions defined on $(\mathbf{w},\mathbf{\tilde{w}}) \in \mathbb{R}^{1 \times s} \times \mathbb{R}^{1 \times s}$ such that the marginal distribution on $\mathbf{w}$ has the density function $p(\cdot)$ and the marginal distribution on $\mathbf{\tilde{w}}$ has the density function $q(\cdot)$. Then,
  \begin{align}\label{tv:cost}
    D_{\mathrm{TV}}(P,Q) =  \inf \limits_{\pi \in \Pi} \mathbb{E}_{\pi}[\mathbf{1}_{\{\mathbf{w}\neq \mathbf{\tilde{w}}\}}]
  \end{align}
   
 \end{definition}
 
Omitted proofs can be found in the \textbf{Supplementary Material}. 
  
\begin{lemma}\label{lem:err}
 (folklore) The classification error of the Bayes optimal classifier $E_{D} [e(f^{*},D)] = \frac{1}{2}-\frac{1}{2} D_{\mathrm{TV}}(P,Q) $.
 \end{lemma}

As a result of Lemma \ref{lem:err} and Lemma \ref{lem:bayesopt}, we have the following corollary:
 \begin{corollary}\label{corollary0}
 $D_{\mathrm{TV}}(P,Q)=1-\int \min (p(\mathbf{u}),q(\mathbf{u})) d\mathbf{u}$. 
 \end{corollary}

 In Definition \ref{def:tvar}, restrict the set of couplings to $\Pi^*$ for which the actual probability space $(\Omega_\pi, {\cal F}_\pi, \pi)$ has the following property:  the sets $\{(\mathbf{u},\mathbf{u}): \mathbf{u} \in B \}$ are measurable with respect to ${\cal F}_\pi$. We have the following technical lemma which will be key in proving the main Theorem \ref{thm:mainthm1}:

\begin{lemma}\label{lem:restrict}
 $D_{\mathrm{TV}}(P,Q) =  \inf \limits_{\pi \in \Pi} \mathbb{E}_{\pi}[\mathbf{1}_{\{\mathbf{w}\neq \mathbf{\tilde{w}}\}}] = \inf \limits_{\pi \in \Pi^*} \mathbb{E}_{\pi}[\mathbf{1}_{\{\mathbf{w}\neq \mathbf{\tilde{w}}\}}]$
 \end{lemma}
   
 \subsection{Analysis of Algorithm \ref{algm:CI-Test}} 
 \label{sec:analysis}
 We consider the performance of the Bayes optimal classifiers for the two binary classification problems: a) Classifying the uniform mixture of $p(\mathbf{x},\mathbf{y},\mathbf{z})$ and $p(\mathbf{z}) q(\mathbf{y}|\mathbf{z}) p(\mathbf{x}|\mathbf{z})$ b) Classifying the uniform mixture of $p(\mathbf{y},\mathbf{z})$ and $q(\mathbf{y}|\mathbf{z}) p(\mathbf{z})$ as in Algorithm \ref{algm:CI-Test}. The basic result of this section is that when the conditionally independent distribution and the conditionally dependent distribution are different, the classification errors of these two classification problems with $q(\mathbf{y}|\mathbf{z})$ exhibit a non-trivial separation. The most interesting point to note is that this happens as long as there is an overlap of support between $q(\mathbf{y}|\mathbf{z})$ and $p(\mathbf{y}|\mathbf{z})$. This means that $q$ need not be close to $p$ in distance but only needs a much weaker condition of support overlap. Let us introduce some notation before we present the result. For every $(\mathbf{y},\mathbf{z})$ consider the conditional density functions $p(\mathbf{x}|\mathbf{z})$ and $p(\mathbf{x}|\mathbf{y},\mathbf{z})$. Let 
\begin{align}\label{eqn:coupling}
    \epsilon(\mathbf{y},\mathbf{z})= \max_{\pi \in \Pi(p(\mathbf{x}|\mathbf{z}),p(\mathbf{x}'|\mathbf{y},\mathbf{z}))} \mathbb{E}_{\pi}[\mathbf{1}_{\{\mathbf{x}=\mathbf{x}'\}}|\mathbf{y}, \mathbf{z}]
\end{align}

Formally, conditional dependence means that $\epsilon(\mathbf{y},\mathbf{z})<1$ with non zero probability with respect to the density function $p(\mathbf{y},\mathbf{z})$. Now, we state the result of this section formally, below:

\begin{theorem}\label{thm:maintv}
  \begin{align}
  & 2\lvert \mathbb{E}_D[e(f^{*}_1,D_s)]- \mathbb{E}_D[e(f^{*}_2,D^{-\mathbf{x}}_s)]  \rvert \nonumber\\
  &= D_{\mathrm{TV}}(p(\mathbf{z},\mathbf{x},\mathbf{y}),p(\mathbf{z}) q(\mathbf{y}|\mathbf{z}) p(\mathbf{x}|\mathbf{z})) - D_{\mathrm{TV}}(p(\mathbf{y},\mathbf{z}),p(\mathbf{z}) q(\mathbf{y}|\mathbf{z})) \nonumber\\
  \hfill & \geq \int \limits_{\mathbf{y},\mathbf{z}} \min (p(\mathbf{z}) q(\mathbf{y}|\mathbf{z}),p(\mathbf{z}) p(\mathbf{y}|\mathbf{z})) (1-\epsilon(\mathbf{y},\mathbf{z})) d(\mathbf{y},\mathbf{z}) 
  \end{align}
\end{theorem}

\textbf{Remark:} One can view Theorem \ref{thm:maintv} as a soft lower bound characterizing the difference when $q(\cdot)$ does not match $p(\cdot)$ perfectly.

\begin{theorem}\label{theorem3}
   As long as the density function $q(\textbf{y}|\mathbf{z})>0$ whenever $p(\mathbf{y},\mathbf{z})>0$, then conditional dependence implies that $2\lvert \mathbb{E}_D[e(f^{*}_1,D_s)]- \mathbb{E}_D[e(f^{*}_2,D^{-\mathbf{x}}_s)]  \rvert>0$.
\end{theorem}

\begin{theorem}\label{theorem4}
  Conditional independence implies that $2\lvert \mathbb{E}_D[e(f^{*}_1,D_s)]- \mathbb{E}_D[e(f^{*}_2,D^{-\mathbf{x}}_s)]  \rvert =0$
\end{theorem}

\textit{Proof of Theorem \ref{thm:mainthm1})} : Combining Theorem \ref{theorem3} and Theorem \ref{theorem4} implies Theorem \ref{thm:mainthm1}.\\

 Intuitively, for the points $(\mathbf{y},\mathbf{z})$ where $\mathbf{x}$ shows strong dependence on both $\mathbf{y}$ and $\mathbf{z}$, one needs $q(\mathbf{y}|\mathbf{z})$ to have large positive density.
As a corollary, we have the following variational characterization of total variation distance between the conditionally dependent and the conditionally independent distributions and another  corollary showing that the bounds for a simple "uniform" mimicking distribution.   
\begin{corollary}\label{corollary1}
$ \max \limits_{q(\cdot)}  \left[D_{\mathrm{TV}}(p(\mathbf{z},\mathbf{x},\mathbf{y}),p(\mathbf{z}) q(\mathbf{y}|\mathbf{z}) p(\mathbf{x}|\mathbf{z})) - D_{\mathrm{TV}}(p(\mathbf{z}) q(\mathbf{y}|\mathbf{z}), p(\mathbf{y},\mathbf{z}))  \right]
= D_{\mathrm{TV}}(p(\mathbf{z},\mathbf{x},\mathbf{y}),p(\mathbf{z}) p(\mathbf{y}|\mathbf{z}) p(\mathbf{x}|\mathbf{z})). 
$
\end{corollary}
  
\begin{corollary}\label{corollary2}
   Suppose $y$ is a scalar and is bounded in the interval $[-b,b]$ with probability $1$ and suppose that $\max \limits_{y,\mathbf{z}} (p(y|\mathbf{z})) \leq a$ for some $a>0$, then the uniform density $q(y|\mathbf{z}) = \frac{1}{2b},~ -b \leq y \leq b$ satisfies the following:   $2\lvert \mathbb{E}_D[e(f^{*}_1,D_s)]- \mathbb{E}_D[e(f^{*}_2,D^{-\mathbf{x}}_s)]  \rvert \geq \frac{1}{2ab} D_{\mathrm{TV}} (p(y,\mathbf{z})p(\mathbf{x}|\mathbf{z}),p(y,\mathbf{z}) p(\mathbf{x}|y,\mathbf{z}))$.
\end{corollary}  

  \subsection{General Measures}
   \label{sec:general-measures}
   So far in our treatment of results, we have assumed that density $p(\mathbf{x}, \mathbf{y}, \mathbf{z})$ exists for the joint distribution of $(\mathbf{X}, \mathbf{Y}, \mathbf{Z})$. Now, let the original probability measure from which the data is generated is given by a general measure $\mathbb{P}$ and a measure $\mathbb{Q}$ is the conditionally independent (due to the markov chain $\mathbf{X}-\mathbf{Z}-\mathbf{Y}$) measure induced on the data set as a result of our algorithm. Note that these two measures differ in the induced conditional measure on $\mathbf{Y}$ given $\mathbf{Z}$ and let $\mathbb{P}_{YZ}$ and $\mathbb{Q}_{YZ}$ be the restrictions of measures on $(\mathbf{Y}, \mathbf{Z})$ respectively. The following theorem generalizes for arbitrary measures:
   
    \begin{theorem}
 \label{thm:mainthm2}
   As long as the the Radon Nikodym derivate $\frac{d\mathbb{Q}_{YZ}}{d\mathbb{P}_{YZ}}$ exists and is $\neq0$ everywhere except set of probability zero with respect to the measure $\mathbb{P}_{YZ}$, $\lvert \mathbb{E}_D[e(f^{*}_1,\tilde{D}_s)]- \mathbb{E}_D[e(f^{*}_2,\tilde{D}^{-\mathbf{x}}_s)]  \rvert =0$ if and only if hypothesis $\mathcal{H}_0$ is true.
   \end{theorem}
   
 \subsection{Finite Samples and p-values}
 \label{sec:pvalue}
The validation set $\tilde{D}_s$ in Algorithm \ref{algm:CI-Test} of size $n$ consists of a labelled uniform mixture of i.i.d samples drawn from the two distributions, $p(\mathbf{z})q(\mathbf{y}|\mathbf{z}) p (\mathbf{x}|\mathbf{z}) $ and $p(\mathbf{z})q(\mathbf{y}|\mathbf{z}) p (\mathbf{x}|\mathbf{y},\mathbf{z}) $.  Let $f_1$ and $f_2$ be two classifiers from the two classification problems from lines \ref{f1_line} and \ref{f2_line} respectively in Algorithm  \ref{algm:CI-Test} . Let $L(f,\mathbf{u},\ell)$ be the zero-one loss function that is $1$ if the prediction $f(\mathbf{u}) \neq \ell$ and $0$ otherwise. Let $\{\mathbf{u}_i, \ell_i \}_{i=1}^n$ denote the labelled examples in $\tilde{D}_s$. Since classifier $f_1$ only operates on other coordinates except the one that contains $\mathbf{x}$, we will assume that it ignores the coordinates corresponding to $\mathbf{x}$ if supplied with those coordinates. Now, we have:
$ e(f_1,\tilde{D}_s^{-\mathbf{x}})-e(f_2,\tilde{D}_s) = \frac{1}{n}\sum \limits_{(\mathbf{u}_i,\ell_i) \in \tilde{D}_s} L(f_1,\mathbf{u}_i,\ell_i)- L(f_2,\mathbf{u}_i,\ell_i).$
Since $(\mathbf{u}_i,\ell_i)$ is i.i.d and $0\leq L(\cdot) \leq 1$ is bounded, $-1 \leq L(f_1,\cdot)-L(f_2,\cdot) \leq 1$. Therefore, applying Chernoff bounds for i.i.d bounded random variables, we have the following subgaussian tail concentration on the test statistic:
 \begin{align}
  P( \lvert e(f_1,\tilde{D}_s^{-\mathbf{x}})-e(f_2,\tilde{D}_s) - \mathbb{E}_D[e(f_1,\tilde{D}_s^{-\mathbf{x}}) -e(f_2,\tilde{D}_s]  \rvert < \epsilon ) \leq 2 \exp(- n\epsilon^2/2),~ \forall \epsilon>0
 \end{align}
 
Supposing, $f_1$ and $f_2$ are Bayes optimal classifiers (i.e. $f^{*}_1$ and $f^{*}_2$ respectively), then $E_D[e(f^{*}_1,\tilde{D}_s^{-\mathbf{x}}) -e(f^{*}_2,\tilde{D}_s)] =0$ if and only if conditional independence holds according to Theorem \ref{thm:mainthm1}. Hence, under the Bayes optimal classification, the tail of the null (and also the non-null) distribution is a subgaussian which can be used to generate a \textbf{p-value.} Further, suppose the Bayes optimal classifiers $f^{*}_1$ and $f^{*}_2$ lie in VC classes of VC dimension at most  $d$. Let $f^{\mathrm{ERM}}_1$ and $f^{\mathrm{ERM}}_2$ be the empirical risk minimizers on the training sets $\tilde{D}_t^{-\mathbf{x}}$ and $\tilde{D}_t$ respectively in Algorithm \ref{algm:CI-Test} over the set of classifiers in their respective VC classes. Then, by the standard results in learning theory  \cite{boucheron_2005}:
\begin{theorem}
\label{thm:pvalue}
 $\forall \delta \in (0,1)$, $\mathbb{E}_D[e(f^{\mathrm{ERM}}_i,\tilde{D}_s)] \leq \mathbb{E}_D[e(f^{*}_i,\tilde{D}_s)] + C \sqrt{\frac{d}{n_t}} + \sqrt{ \frac{2\log(1/\delta)}{n_t}} $ with probability atleast $1- \delta$ where $n_t=|\tilde{D}_t|$.
\end{theorem}

\textbf{Note:} The analysis above does not use the fact that we are restricted to distributions with densities, hence the same results of Theorem \ref{thm:pvalue} hold when we have general measures as in the case of Section \ref{sec:general-measures}.

\section{Candidate \texttt{MIMIC} Functions}
\label{sec:mimic}

\textbf{Conditional GAN:}
A promising candidate for mimicking samples from $p(\y \vert \z)$ is using conditional generative adversarial networks (CGAN)~\cite{mirza2014conditional}. Following the methodology in~\cite{mirza2014conditional}, we adversarially train a generator deep network $G(\z,\mathbf{s}) \in \mathbb{R}^{n_y}$ and a discriminator network $D(\y,\z) \in [0,1]$. Here, $\mathbf{s} \sim p(\mathbf{s})$ is a standard normal noise random variable of dimension $d_s$, which is specified in our experiments. The goal is to obtain a generator function $G(\z,\mathbf{s})$ such that given a $\z$, the distribution of $G(\z,\mathbf{s})$ over the randomness in $\mathbf{s}$ closely resembles that of $p(\y \vert \z)$. Once we have such a generator, given any $\mathbf{z}$ we can mimic samples from  $p(\y \vert \z)$, by randomly generating a noisy random variable $\mathbf{s}$ and evaluating the generator at the point $(\z,\mathbf{s})$. Once we have this neural network based mimic function $G(\z,\mathbf{s})$ we can develop a CI-Test using our Mimic and Classify philosophy.  This CI-Test is dubbed \textbf{MIMIFY-CGAN}. Additional details of the exact methodology can be found in the supplement.



\textbf{Regression based approaches:}
Now we introduce our second \texttt{MIMIC} approach which is based on the idea of regression. Given samples from a joint distribution $p(\y,\z)$, one can form a regression problem where the aim is to predict $\y$ given $\z $. This is classically solved by training a function $r:\mathbb{R}^{d_z} \rightarrow \mathbb{R}^{d_y}$ (in a class of functions $\mathcal{R}$), such that ideally $\eta(\y,\z,r) = \EE_{\Y,\Z \sim p(\y,\z)}[(Y - r(Z))^2 ]$ is minimized. This done by minimizing an empirical estimate of $\eta(\y,\z,r)$ by using the samples from the joint distribution. The function that minimizes the mean squared error $\eta(.)$ globally is the conditional expectation $\EE[\Y \vert \Z]$. Modern regression classes like gradient boosted trees and deep networks are extremely powerful and can fit the conditional expectation $\EE[\Y \vert \Z]$ very closely in most cases. This leads us to our regression based \texttt{MIMIC} function. The idea is to train a regression function $r(\z)$ that closely mimics the conditional expectation $\EE[\Y|\Z = \z]$, for any $\Z = \z$, given a data-set of samples from the joint $p(\y,\z)$. This can be done using standard regression techniques. Now, given any $\z$, we can evaluate $r(z)$ which resembles $\EE[\Y|\Z = \z]$. We can add a noise random variable $\mathbf{s} \in \mathbb{R}^{n_y}$ to create the variable $\hat{\y} = r(\z) + \mathbf{s}$. The random variable $\hat{\y}$ is thus centered approximately at $\EE[\Y|\Z = \z]$ and if the noise random variable is chosen appropriately, then the distribution of $\hat{\y}$ (denoted by say $q_{r}(\y | \z)$) is bound to have a significant overlap with $p(\y | \z)$, especially if the true conditional $p(\y | \z)$ is unimodal. This is precisely our \texttt{MIMIC} function. Our regression based \texttt{MIMIC} function is dubbed \textbf{MIMIFY-REG}. The exact methodology is described in the supplement.

\section{Empirical Results}
\label{sec:sims}
In this section we empirically validate our algorithms against state of the art methods, on synthetic and real data-sets. The algorithms under contention are: $(i)$ MIMIFY-CGAN: Our CGAN based \texttt{MIMIC} and \texttt{CLASSIFY} method. We use a fully connected generator with $2$ hidden layers. The discriminator is also fully connected with two hidden layers, the last layer being a sigmoid layer. The training is done according to the method followed in~\cite{mirza2014conditional}. The noise dimension is set to $d_{s} = 20$ in all our experiments, $(ii)$ MIMIFY-REG: Our regression based method. We use the XGB-Regressor in the scikit-learn API of XGBoost~\cite{chen2016xgboost} as our regression function. The classifier used in our \texttt{CLASSIFY} phase is also XGBoost, $(iii)$ KCIT~\cite{kcit_2011}: We use the implementation in the RCIT R package provided by the authors of~\cite{rcit_2017}, $(iv)$ RCIT~\cite{rcit_2017}: We use the implementation in the RCIT R packages, and $(v)$ CCIT~\cite{ccit_2017}: We use the python package provided by the authors of~\cite{ccit_2017}.\footnote{The software package for our implementation can be found here: (https://github.com/rajatsen91/mimic\_classify)}

{\bf Post-Nonlinear Noise Synthetic Experiments: } We test all our algorithms on a harder version of the post-nonlinear noise setting that has been used in~\cite{kcit_2011,ccit_2017,rcit_2017}. Motivated by applications in causal inference, in all our experiments $d_x = d_y = 1$ while the dimension of $z$ can scale. In our synthetic datasets, when the ground truth is $\X \ci \Y \vert \Z $ then the variables follows the relation $\X = f_1(A_x\Z + \eta_1)$ and $\Y = f_2(A_y\Z + \eta_2)$. When the ground truth is not CI, then $\Y = f_2(A_y\Z + a_{xy} \X + \eta_2)$ where $a_{xy}$ is a fixed constant.  Here, $A_x,A_y \in \mathbb{R}^{d_z \times 1}$ are matrices that are held fixed for generating the samples of a single data-set. $\eta_1$ and $\eta_2$ are zero-mean Gaussian noise of variance $0.25$. For each data-set $f_1$ and $f_2$ are non-linear functions chosen at random from the set of functions $\{x,x^2,x^3, \tanh(x), \exp(-x) \}$, for each data-set. 

 We plot the ROC-AUC achieved by the different algorithms as a function of $d_z$ in Fig.~\ref{fig:dim}. Each point in the plot involved generating $50$ data-sets in which $\X \ci \Y \vert \Z $ and $50$ data-sets where $\X$ is not independent of $\Y$ given $\Z $, each data-set containing $5000$ samples. The algorithms are run on each of the data-sets and then ROC-AUC is calculated by using the p-values generated and the ground-truth labels for each data-set. RCIT performs much worse than the other algorithms on this data-set and therefore has been omitted from the plot. It can be seen for the extreme case of $d_z = 200$, MIMIFY-GAN beats the other algorithms in terms of ROC-AUC. This can be attributed to the fact that CGANS can in fact fit $p(\y | \z)$ even when $\z$ is high-dimensional. MIMIFY-GAN achieved an ROC-AUC of $0.835$ even when $d_z = 300$. KCIT could not be run due to high run-times at this scale. 

\begin{figure*}
	\centering
	
	\subfloat[][]{\includegraphics[width = 0.31\linewidth]{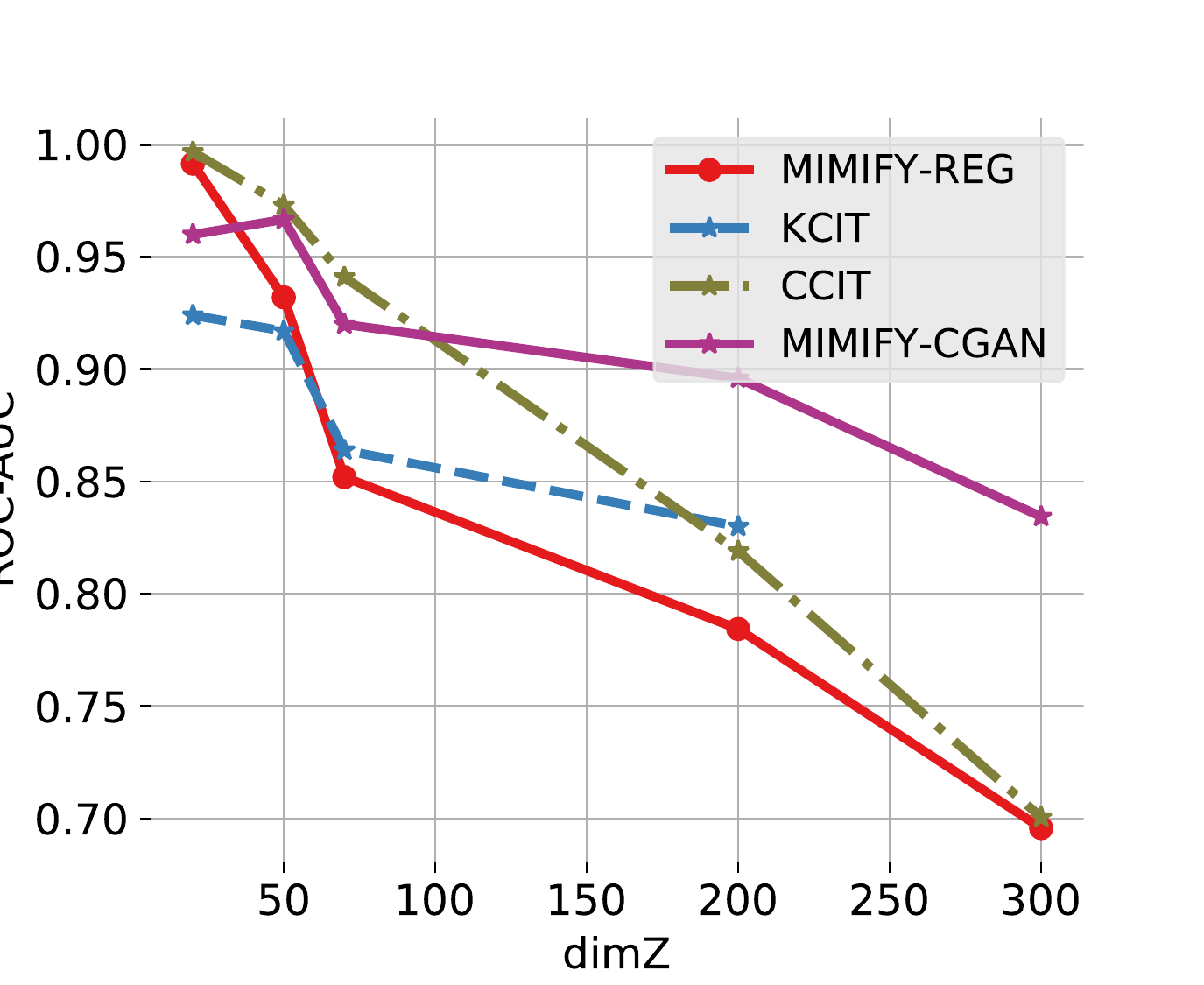}\label{fig:dim}} \hfill
	\subfloat[][]{\includegraphics[width = 0.31\linewidth]{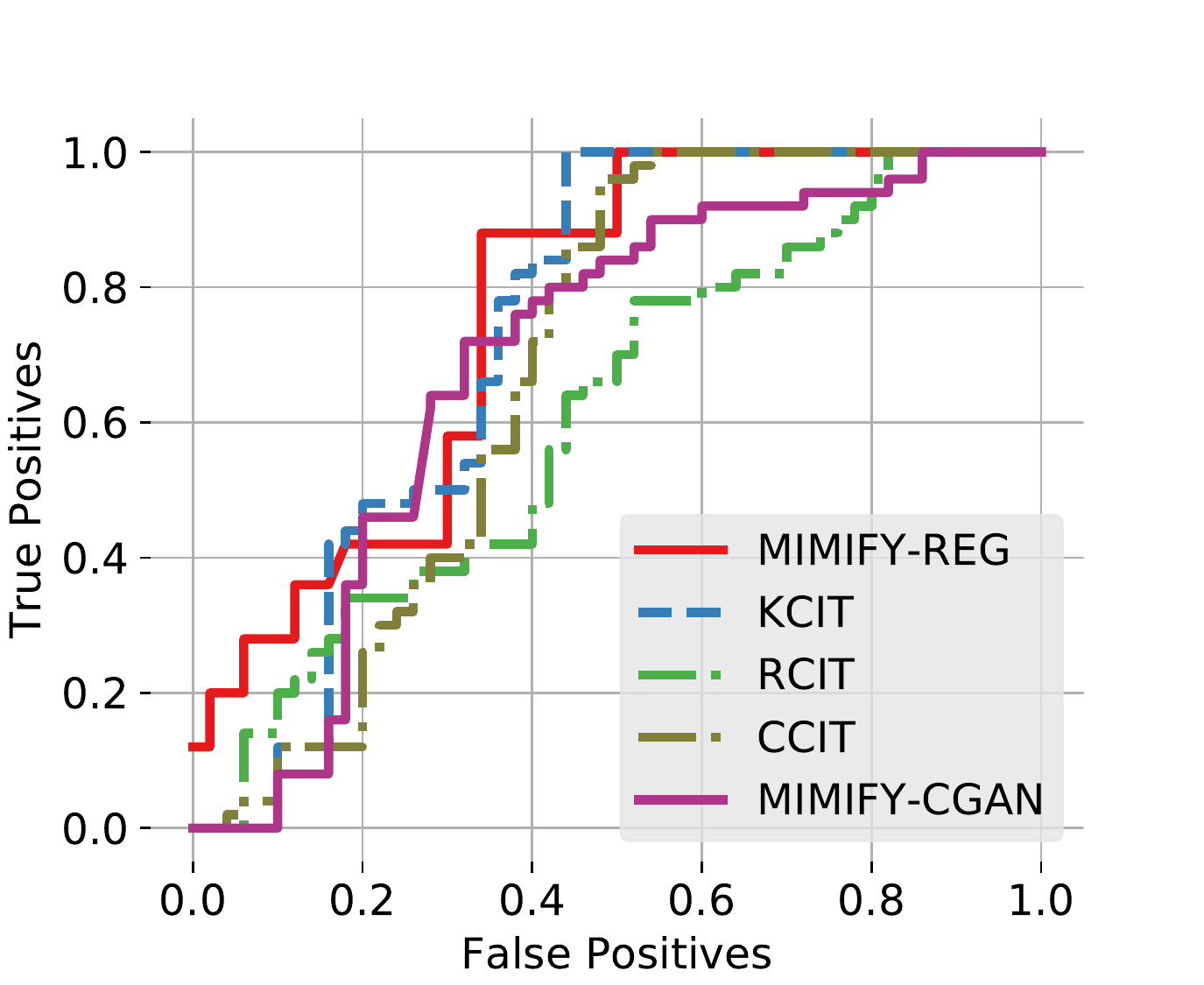}\label{fig:cyto}} \hfill
	\subfloat[][]{\includegraphics[width = 0.31\linewidth]{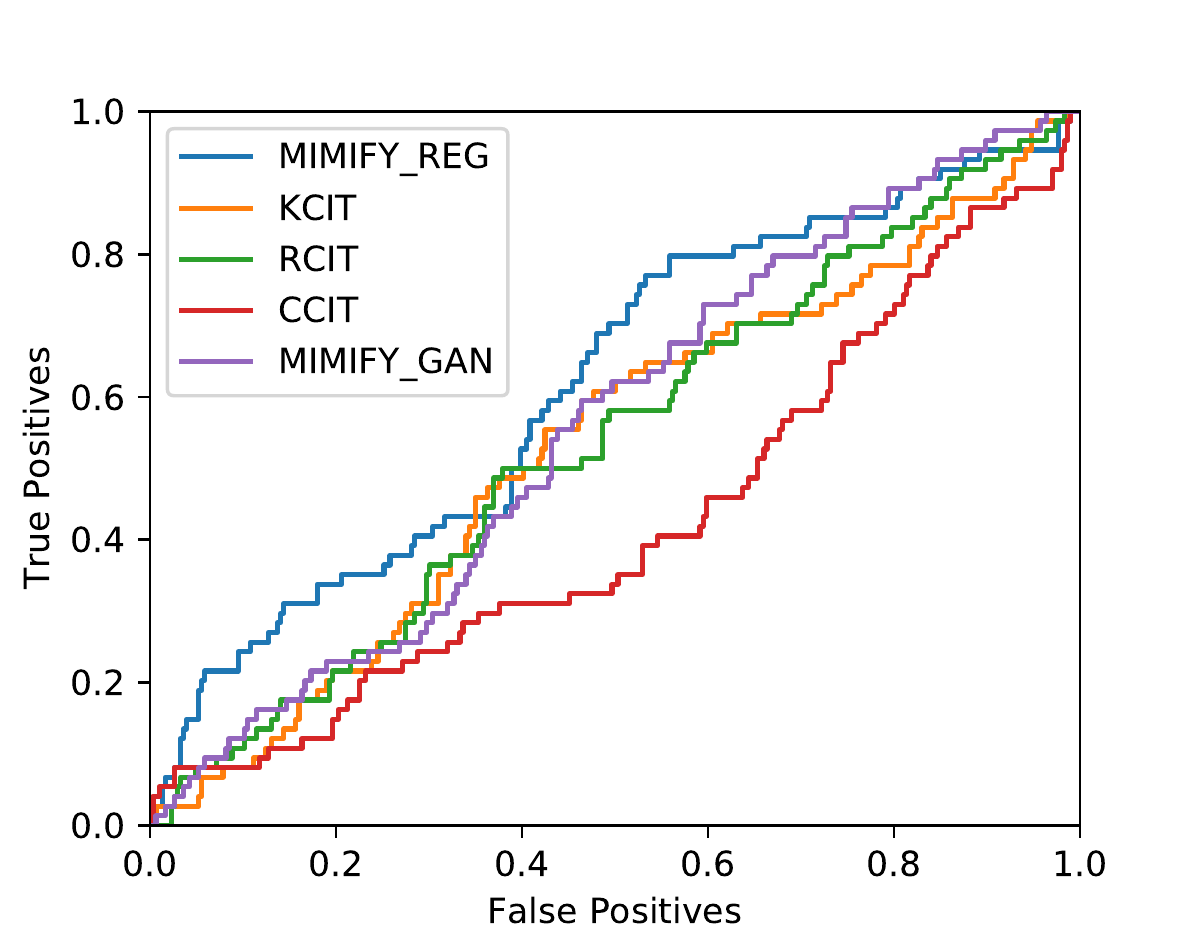}\label{fig:grn_roc}} \hfill
	\subfloat[][]{
		
		\resizebox{0.31\linewidth}{!}{
			\begin{tabular}[b]{ lccc }
				\hline
				Algo. & ROC-AUC\\ \midrule 
				MIMIFY-REG & {\color{red} 0.7638} \\ 
				KCIT & 0.7328 \\
				MIMIFY-GAN & {\color{blue} 0.6891} \\ 
				CCIT & 0.6816 \\
				RCIT & 0.6135\\\bottomrule
				\hline
			\end{tabular}
		}	
		\label{fig:tab_cytometry}
		}
	\subfloat[][]{
		
		\resizebox{0.31\linewidth}{!}{
			\begin{tabular}[b]{ lccc }
				\hline
				Algo. & ROC-AUC\\ \midrule 
				MIMIFY-REG & {\color{red} 0.61645} \\ 
				MIMIFY-GAN & {\color{blue} 0.55679} \\
				RCIT & 0.53511\\
				KCIT & 0.53175 \\
				CCIT & 0.42439 \\ \bottomrule
				\hline
			\end{tabular}
		}
		\label{fig:tab_grn}
		
	} 
	\caption{ \small In (a) we plot the performance of our algorithms in the post-nonlinear noise synthetic data. In generating each point in the plots, $100$ data-sets are generated where half of them are according to $\mathcal{H}_0$ while the rest are according to $\mathcal{H}_1$. The algorithms are run on each of them, and the ROC AUC score is plotted. The number of samples is $n = 1000$, while the dimension of $Z$ varies. RCIT is omitted in this plot, as it performs poorly and outside the range of values plotted. In $(b)$ we plot the ROC curve for all three algorithms based on the data from for the flow-cytometry dataset. In $(c)$ ROC curves for gene regulatory network inference from DREAM dataset with various CIT methods are shown. The ROC AUC score for each of the algorithms in cytometry and gene regulatory networks is provided in $(d)$ and $(e)$ respectively, considering their respective ground-truth causal graphs. }
	\label{fig:cytoresults}
\end{figure*}
 
{\bf Flow-Cytometry Data: } We use our CI testing algorithm to verify CI relations in the protein network data from the flow-cytometry dataset~\cite{sachs2005causal}, which gives expression levels of $11$ proteins under various experimental conditions. We use the reconstructed causal graph in~\cite{sachs2005causal} (Fig. 1(b) in~\cite{mooij2013cyclic}) as the ground truth network. The CI relations are generated as follows: for each node $\X$ in the graph, identify the set $\Z$ consisting of its parents, children and parents of children in the causal graph. Conditioned on this set $\Z$, $\X$ is independent of every other node $\Y$ in the graph (apart from the ones in $\Z$). We use this to create $50$ randomly chosen CI relations. In order to evaluate false positives of our algorithms, we also need relations such that $\X \centernot{\ci} \Y \vert \Z$. For, this we observe that if there is an edge between two nodes, they are never CI given any other conditioning set. For each graph we generate $50$ such non-CI relations, where an edge $\X \leftrightarrow \Y$ is selected at random and a conditioning set of size $5$ is randomly selected from the remaining nodes. We construct $50$ such negative examples for each graph. For the sake of reproducibility we include all these relations as a csv file in our supplementary,  where the first $50$ relations are CI and the rest are not CI. The column $X$, $Y$ denote the nodes and $Z$ the conditioning set. The ROC (Receiver Operating Characteristics) curve is plotted from the results of all the algorithm in Fig.~\ref{fig:cyto}. The corresponding ROC-AUC achieved are given in Table~\ref{fig:tab_cytometry}. MIMIFY-REG outperforms the other algorithms on this data-set achieving an ROC-AUC of $0.7638$.

{\bf Gene Regulatory Network Inference for DREAM dataset: } In this experiment, we test and evaluate different algorithms over the \textit{in silico} cell development process dataset taken from~\cite{hill2016inferring}. The dataset represents a cell development process including 20 genes, which consists of 60 separate experiments each involving time-series of the length 11. Therefore the total number of samples is 660. The ground-truth regulatory network is known. Form the process, we created a matrix of pairwise CI p-values in the form of $CIT( X_i, X_j | \{ X_i, X_j \}^c )$ given the samples, and then evaulated the performance of each algorithm through their respective ROC curves. The plotted ROC curves and the obtained AUC values are presented in Fig.~\ref{fig:grn_roc} and Table~\ref{fig:tab_grn} respectively.

\section{Conclusion and Future Work}
The paradigm of \textbf{Mimic and Classify} is proposed as per \textit{meta}-algorithm \ref{algm:CI-Test} with provable guarantees on the null distribution and analysis for general probability measures, thus suggestive of a two-step approach: (a) \textit{Mimic} the CI Distribution and (b) \textit{Classify} to distinguish the joint and CI distribution. New candidates are discovered such as cGANs and Regression based methods, qualifying for mimicking step and they outperform on experiments. Future directions include discovering efficient candidates for both mimic and classify steps to boost CI Testing performance especially in high dimensional and moderate to small sample regime. Bolstered by strong CI Testing performance, we would like to see their utility in causal inference problems. 

\bibliographystyle{plain}
\bibliography{CI-ref.bib}

\begin{thebibliography}{10}

\bibitem{arora_gan_2_2018}
Sanjeev Arora, Rong Ge, Yingyu Liang, Tengyu Ma, and Yi~Zhang.
\newblock Generalization and equilibrium in generative adversarial nets
  ({GAN}s).
\newblock In Doina Precup and Yee~Whye Teh, editors, {\em Proceedings of the
  34th International Conference on Machine Learning}, volume~70 of {\em
  Proceedings of Machine Learning Research}, pages 224--232, International
  Convention Centre, Sydney, Australia, 06--11 Aug 2017. PMLR.

\bibitem{arora_gan_1_2018}
Sanjeev Arora, Andrej Risteski, and Yi~Zhang.
\newblock Do {GAN}s learn the distribution? some theory and empirics.
\newblock In {\em International Conference on Learning Representations}, 2018.

\bibitem{boucheron_2005}
Stephane Boucheron, Olivier Bousquet, and Gabor Lugosi.
\newblock Theory of classification : a survey of some recent advances.
\newblock {\em ESAIM: Probability and Statistics}, 9:323--375, 2005.

\bibitem{xgboost_2016}
Tianqi Chen and Carlos Guestrin.
\newblock Xgboost: A scalable tree boosting system.
\newblock In {\em Proceedings of the 22Nd ACM SIGKDD International Conference
  on Knowledge Discovery and Data Mining}, KDD '16, pages 785--794, New York,
  NY, USA, 2016. ACM.

\bibitem{chen2016xgboost}
Tianqi Chen and Carlos Guestrin.
\newblock Xgboost: A scalable tree boosting system.
\newblock In {\em Proceedings of the 22nd acm sigkdd international conference
  on knowledge discovery and data mining}, pages 785--794. ACM, 2016.

\bibitem{daudin_1980}
J.~J. DAUDIN.
\newblock Partial association measures and an application to qualitative
  regression.
\newblock {\em Biometrika}, 67(3):581--590, 1980.

\bibitem{pcit_2014}
Gary Doran, Krikamol Muandet, Kun Zhang, and Bernhard Sch\"{o}lkopf.
\newblock A permutation-based kernel conditional independence test.
\newblock In {\em Proceedings of the Thirtieth Conference on Uncertainty in
  Artificial Intelligence}, UAI'14, pages 132--141, Arlington, Virginia, United
  States, 2914. AUAI Press.

\bibitem{fukumizu_2004}
Kenji Fukumizu, Francis~R. Bach, and Michael~I. Jordan.
\newblock Dimensionality reduction for supervised learning with reproducing
  kernel hilbert spaces.
\newblock {\em J. Mach. Learn. Res.}, 5:73--99, December 2004.

\bibitem{fukumizu_2008}
Kenji Fukumizu, Arthur Gretton, Xiaohai Sun, and Bernhard Sch\"{o}lkopf.
\newblock Kernel measures of conditional dependence.
\newblock In J.~C. Platt, D.~Koller, Y.~Singer, and S.~T. Roweis, editors, {\em
  Advances in Neural Information Processing Systems 20}, pages 489--496. Curran
  Associates, Inc., 2008.

\bibitem{pramod_demystify_2018}
W.~Gao, S.~Oh, and P.~Viswanath.
\newblock Demystifying fixed k-nearest neighbor information estimators.
\newblock {\em IEEE Transactions on Information Theory}, pages 1--1, 2018.

\bibitem{gao2016conditional}
Weihao Gao, Sreeram Kannan, Sewoong Oh, and Pramod Viswanath.
\newblock Conditional dependence via shannon capacity: Axioms, estimators and
  applications.
\newblock {\em arXiv preprint arXiv:1602.03476}, 2016.

\bibitem{gao2017mixture}
Weihao Gao, Sreeram Kannan, Sewoong Oh, and Pramod Viswanath.
\newblock Estimating mutual information for discrete-continuous mixtures.
\newblock In {\em Advances in Neural Information Processing Systems}, pages
  5988--5999, 2017.

\bibitem{dl_book_2016}
Ian Goodfellow, Yoshua Bengio, and Aaron Courville.
\newblock {\em Deep Learning}.
\newblock The MIT Press, 2016.

\bibitem{gan_2014}
Ian Goodfellow, Jean Pouget-Abadie, Mehdi Mirza, Bing Xu, David Warde-Farley,
  Sherjil Ozair, Aaron Courville, and Yoshua Bengio.
\newblock Generative adversarial nets.
\newblock In Z.~Ghahramani, M.~Welling, C.~Cortes, N.~D. Lawrence, and K.~Q.
  Weinberger, editors, {\em Advances in Neural Information Processing Systems
  27}, pages 2672--2680. Curran Associates, Inc., 2014.

\bibitem{gretton_2008}
Arthur Gretton, Kenji Fukumizu, Choon~H. Teo, Le~Song, Bernhard Sch\"{o}lkopf,
  and Alex~J. Smola.
\newblock A kernel statistical test of independence.
\newblock In J.~C. Platt, D.~Koller, Y.~Singer, and S.~T. Roweis, editors, {\em
  Advances in Neural Information Processing Systems 20}, pages 585--592. Curran
  Associates, Inc., 2008.

\bibitem{hill2016inferring}
Steven~M Hill, Laura~M Heiser, Thomas Cokelaer, Michael Unger, Nicole~K Nesser,
  Daniel~E Carlin, Yang Zhang, Artem Sokolov, Evan~O Paull, Chris~K Wong,
  et~al.
\newblock Inferring causal molecular networks: empirical assessment through a
  community-based effort.
\newblock {\em Nature methods}, 13(4):310, 2016.

\bibitem{huang_2010}
Tzee-Ming Huang.
\newblock Testing conditional independence using maximal nonlinear conditional
  correlation.
\newblock {\em Ann. Statist.}, 38(4):2047--2091, 08 2010.

\bibitem{koller_book}
Daphne Koller and Nir Friedman.
\newblock {\em Probabilistic Graphical Models: Principles and Techniques -
  Adaptive Computation and Machine Learning}.
\newblock The MIT Press, 2009.

\bibitem{kozachenko1987sample}
LF~Kozachenko and Nikolai~N Leonenko.
\newblock Sample estimate of the entropy of a random vector.
\newblock {\em Problemy Peredachi Informatsii}, 23(2):9--16, 1987.

\bibitem{kraskov2004estimating}
Alexander Kraskov, Harald St{\"o}gbauer, and Peter Grassberger.
\newblock Estimating mutual information.
\newblock {\em Physical review E}, 69(6):066138, 2004.

\bibitem{imagenet_2012}
Alex Krizhevsky, Ilya Sutskever, and Geoffrey~E. Hinton.
\newblock Imagenet classification with deep convolutional neural networks.
\newblock In {\em Proceedings of the 25th International Conference on Neural
  Information Processing Systems - Volume 1}, NIPS'12, pages 1097--1105, USA,
  2012. Curran Associates Inc.

\bibitem{margaritis_2005}
Dimitris Margaritis.
\newblock Distribution-free learning of bayesian network structure in
  continuous domains.
\newblock In {\em Proceedings, The Twentieth National Conference on Artificial
  Intelligence and the Seventeenth Innovative Applications of Artificial
  Intelligence Conference, July 9-13, 2005, Pittsburgh, Pennsylvania, {USA}},
  pages 825--830, 2005.

\bibitem{cgan_2014}
Mehdi Mirza and Simon Osindero.
\newblock Conditional generative adversarial nets.
\newblock {\em CoRR}, abs/1411.1784, 2014.

\bibitem{mirza2014conditional}
Mehdi Mirza and Simon Osindero.
\newblock Conditional generative adversarial nets.
\newblock {\em arXiv preprint arXiv:1411.1784}, 2014.

\bibitem{mooij2013cyclic}
Joris Mooij and Tom Heskes.
\newblock Cyclic causal discovery from continuous equilibrium data.
\newblock {\em arXiv preprint arXiv:1309.6849}, 2013.

\bibitem{pearl_book}
Judea Pearl.
\newblock {\em Causality: Models, Reasoning and Inference}.
\newblock Cambridge University Press, New York, NY, USA, 2nd edition, 2009.

\bibitem{peters_book}
J.~Peters, D.~Janzing, and B.~Sch\"olkopf.
\newblock {\em Elements of Causal Inference: Foundations and Learning
  Algorithms}.
\newblock MIT Press, Cambridge, MA, USA, 2017.

\bibitem{cmit_2018}
Jakob Runge.
\newblock Conditional independence testing based on a nearest-neighbor
  estimator of conditional mutual information.
\newblock In Amos Storkey and Fernando Perez-Cruz, editors, {\em Proceedings of
  the Twenty-First International Conference on Artificial Intelligence and
  Statistics}, volume~84 of {\em Proceedings of Machine Learning Research},
  pages 938--947, Playa Blanca, Lanzarote, Canary Islands, 09--11 Apr 2018.
  PMLR.

\bibitem{sachs2005causal}
Karen Sachs, Omar Perez, Dana Pe'er, Douglas~A Lauffenburger, and Garry~P
  Nolan.
\newblock Causal protein-signaling networks derived from multiparameter
  single-cell data.
\newblock {\em Science}, 308(5721):523--529, 2005.

\bibitem{sason2013entropy}
Igal Sason.
\newblock Entropy bounds for discrete random variables via maximal coupling.
\newblock {\em IEEE Transactions on Information Theory}, 59(11):7118--7131,
  2013.

\bibitem{ccit_2017}
Rajat Sen, Ananda~Theertha Suresh, Karthikeyan Shanmugam, Alexandros~G Dimakis,
  and Sanjay Shakkottai.
\newblock Model-powered conditional independence test.
\newblock In I.~Guyon, U.~V. Luxburg, S.~Bengio, H.~Wallach, R.~Fergus,
  S.~Vishwanathan, and R.~Garnett, editors, {\em Advances in Neural Information
  Processing Systems 30}, pages 2951--2961. Curran Associates, Inc., 2017.

\bibitem{singh2003nearest}
Harshinder Singh, Neeraj Misra, Vladimir Hnizdo, Adam Fedorowicz, and Eugene
  Demchuk.
\newblock Nearest neighbor estimates of entropy.
\newblock {\em American journal of mathematical and management sciences},
  23(3-4):301--321, 2003.

\bibitem{spirtes_book}
P.~Spirtes, C.~Glymour, and R.~Scheines.
\newblock {\em Causation, Prediction, and Search}.
\newblock MIT Press, Cambridge, MA, USA, 2000.

\bibitem{rcit_2017}
E.~V. {Strobl}, K.~{Zhang}, and S.~{Visweswaran}.
\newblock {Approximate Kernel-based Conditional Independence Tests for Fast
  Non-Parametric Causal Discovery}.
\newblock {\em ArXiv e-prints}, February 2017.

\bibitem{su_white_2007}
Liangjun Su and Halbert White.
\newblock {A consistent characteristic function-based test for conditional
  independence}.
\newblock {\em Journal of Econometrics}, 141(2):807--834, December 2007.

\bibitem{su_white_2008}
Liangjun Su and Halbert White.
\newblock A nonparametric hellinger metric test for conditional independence.
\newblock {\em Econometric Theory}, 24(04):829--864, 2008.

\bibitem{tong2000restricted}
Simon Tong and Daphne Koller.
\newblock Restricted bayes optimal classifiers.
\newblock In {\em AAAI/IAAI}, pages 658--664, 2000.

\bibitem{cdit_2015}
Xueqin Wang, Wenliang Pan, Wenhao Hu, Yuan Tian, and Heping Zhang.
\newblock Conditional distance correlation.
\newblock {\em Journal of the American Statistical Association},
  110(512):1726--1734, 2015.
\newblock PMID: 26877569.

\bibitem{kcit_2011}
Kun Zhang, Jonas Peters, Dominik Janzing, and Bernhard Sch\"{o}lkopf.
\newblock Kernel-based conditional independence test and application in causal
  discovery.
\newblock In {\em Proceedings of the Twenty-Seventh Conference on Uncertainty
  in Artificial Intelligence}, UAI'11, pages 804--813, Arlington, Virginia,
  United States, 2011. AUAI Press.

\end{thebibliography}


\newpage
\begin{appendices}

\section{Exact Methodology of \textbf{MIMIFY-GAN}}

 The exact methodology followed for \textbf{MIMIFY-GAN} is as follows: $(i)$ Given $n$ samples $\{\x_i,\y_i,\z_i\}$ ($i = 1,..,n$) we randomly subdivide the samples into three disjoint sets $\mathcal{D}_1$, $\mathcal{D}_2$ and $\mathcal{D}_3$, with $|\mathcal{D}_i| = n/3$. The samples in $\mathcal{D}_1$ are kept as it is and labeled $1$, $(ii)$ A CGAN $G(\z,\mathbf{s})$ is trained using the $\y,\z$ coordinates of the samples in $\mathcal{D}_2$ in order to mimic $p(\y\vert \z)$, $(iii)$ For every sample $(\x,\y,\z) \in \mathcal{D}_3$ we create a new sample $(\x,\hat{\y},\z)$ where $\hat{\y} = G(\z,\mathbf{s})$ and $\mathbf{s}$ is randomly generated from $p(\mathbf{s})$. These new samples $\{(\x,\hat{\y},\z)\}$ are labeled $0$ and added to the data-set of previously labeled samples in step $(i)$. Thus we have a labeled classification data-set. $(iv)$ We randomly subdivide the labeled data-set into training and test sets. Now the steps 5 - 10 of Algorithm~\ref{algm:CI-Test} can be followed using a standard classifier (XGBoost~\cite{chen2016xgboost} in our implementation) in order to obtain a hypothesis. 

\section{Exact Methodology of \textbf{MIMIFY-REG}}
 The exact methodology for \textbf{MIMIFY-REG} is as follows: $(i)$ Given $n$ samples $\{\x_i,\y_i,\z_i\}$ ($i = 1,..,n$) we randomly subdivide the samples into three disjoint sets $\mathcal{D}_1$, $\mathcal{D}_2$ and $\mathcal{D}_3$, with $|\mathcal{D}_i| = n/3$. The samples in $\mathcal{D}_1$ are kept as it is and labeled $1$, $(ii)$ We fit a regression function (XGBoost~\cite{chen2016xgboost} regressor in our implementation) to predict $\y$ given $\z$, using the samples in $\mathcal{D}_2$. Let $r(.)$ denote the trained regression function, $(iii)$ For every sample $(\x,\y,\z) \in \mathcal{D}_3$ we create a new sample $(\x,\hat{\y},\z)$ where $\hat{\y} = r(\z) + s$ and $\mathbf{s} \in \mathbb{R}^{d_y}$ is a random noise. In order to keep our model versatile, we use two noise models for generating $\mathbf{s}$. Our first noise model is where $\mathbf{s}$ is a multi-variate \textit{Gaussian} vector with zero mean and covariance equal to that of the residual vector $\y_r = \y - r(\z)$, measured empirically from the samples in $\mathcal{D}_2$. In our second noise model, $\mathbf{s}$ is such that each coordinate is a Laplace random variable. The variance of the $i$-th coordinate is set equal to that of the $i$-th coordinate of $\y_r$ measured from the empirical data in $\mathcal{D}_2$. While processing every sample $(\x,\y,\z) \in \mathcal{D}_3$, with a probability of $0.3$ we add Gaussian noise, otherwise we add Laplace noise. Thus, we have created a mimicked data-set of $n/3$ samples $\{(\x,\hat{\y},\z)\}$, all of which are labeled $0$. These labeled samples are added to the original samples in step $(i)$ yielding our classification data-set. $(iv)$ We randomly subdivide the labeled data-set into training and test sets. Now the steps 5 - 10 of Algorithm~\ref{algm:CI-Test} can be followed using a standard classifier (XGBoost~\cite{chen2016xgboost} in our implementation) in order to obtain a hypothesis.

\section {Relation between Bayes Optimal Classifier and Total Variation Distance }
\subsection{Proof of Lemma \ref{lem:err}}

   The proof follows from Lemma (\ref{lem:bayesopt}) and Eq. (\ref{tv:int}), by using a simple algebraic identity, $\min(a,b) = \frac{a+b-|a-b|}{2}$. \qed

\subsection{Proof of Lemma \ref{lem:restrict}}
For any coupling $\pi \in \Pi^*$, for all measurable sets $B$ in ${\cal F}$, we define a measure $\nu_{\pi}$ with sample space $\mathbb{R}^{1 \times s}$ and sigma algebra ${\cal F}$ as follows:
  $\nu_{\pi} (B) = \mathbb{E}_{\pi} [ \mathbf{1}_{\{\mathbf{w}=\mathbf{\tilde{w}}= \mathbf{u}\}} \mathbf{1}_{\{\mathbf{u} \in B\}}]$. 
Note that, this definition is possible because of the assumption that $\{(\mathbf{u},\mathbf{u}):\mathbf{u} \in B\}$ is measurable in ${\cal F}_{\pi}$ for every  $B \in {\cal F}$, for every coupling $\pi\in\Pi^*$. Let the marginal measures of $\pi$ be $P_\pi$ and $Q_\pi$ with densities $p_\pi$ and $q_\pi$ , respectively. 

Consider a set $B \in {\cal F}$ that has zero Lebesgue measure. This implies that both $P_\pi$ and $Q_\pi$ have zero measure on $B$. We observe that $\nu_\pi(B) \leq P_\pi(B), Q_\pi(B)$ for all measurable $B \in {\cal F}$ because the marginal measures of $\pi$ are $P_\pi$ and $Q_\pi$, proving that $\nu_{\pi}(B)=0$. By the Radon-Nikodym theorem, there is a measurable density function $g_{\pi}(\mathbf{u})$ such that: 
  \begin{align}\label{eqn:gdensity}
  \nu_{\pi}(B) = \int \limits_{B} g_{\pi}(\mathbf{u}) d\mathbf{u}, ~ \forall B \in {\cal F}, \forall \pi\in\Pi^* .
  \end{align}

  We already observed that $\nu_\pi(B) \leq P_\pi(B), Q_\pi(B)$ for all measurable $B \in {\cal F}$. And all the measures $\nu_{\pi},P_\pi$ and $Q_\pi$ are absolutely continuous with respect to the Lebesgue measure. This implies that their density functions satisfy the inequalities almost surely. That is, 
  \begin{align}\label{eqn:gdensityg}
  g_{\pi}(\mathbf{u}) \leq \min (p_\pi(\mathbf{u}),q_\pi(\mathbf{u})) \text{ \textit{a.s.}}
  \end{align} 
  
  Now, from the definition of total variation distance in Eq. (\ref{tv:cost}) note that:
  \begin{align}
  D_{\mathrm{TV}}(P,Q) &=  \inf \limits_{\pi \in \Pi} \mathbb{E}_{\pi}[\mathbf{1}_{\{\mathbf{w}\neq \mathbf{\tilde{w}}\}}] \nonumber\\
  & \le  \inf \limits_{\pi \in \Pi^*} \mathbb{E}_{\pi}[\mathbf{1}_{\{\mathbf{w}\neq \mathbf{\tilde{w}}\}}]\nonumber\\
  & \overset{(a)}{\le} \mathbb{E}_{\pi}[\mathbf{1}_{\{\mathbf{w}\neq \mathbf{\tilde{w}}\}}]\nonumber\\
  & = \nu_{\pi}(\Omega_\pi) \nonumber\\
  & \overset{(b)}{\le} \int \min (p(\mathbf{u}),q(\mathbf{u})) d\mathbf{u}\label{eq:dtvmin}
  \end{align}
  where in (a) $\pi$ is some coupling in $\Pi^*$, while (b) follows using the Eq. (\ref{eqn:gdensity}) and (\ref{eqn:gdensityg}). The equality in Eq. (\ref{eq:dtvmin}) can be proven by showing a coupling that satisfies the inequality exactly. Such a coupling has been constructed for discrete pmfs $p$ and $q$ in \cite{sason2013entropy}(Theorem $1$). The exact same construction can be extended by using the density functions $p$ and $q$ (as in Eqs. $(6), (7)$ and $(8)$ in \cite{sason2013entropy}).  Since inequalities (a) and (b) hold with equality, we have:
   \begin{align}
 D_{\mathrm{TV}}(P,Q) &=  \inf \limits_{\pi \in \Pi} \mathbb{E}_{\pi}[\mathbf{1}_{\{\mathbf{w}\neq \mathbf{\tilde{w}}\}}] \nonumber\\
  & = \inf \limits_{\pi \in \Pi^*} \mathbb{E}_{\pi}[\mathbf{1}_{\{\mathbf{w}\neq \mathbf{\tilde{w}}\}}]\nonumber
 \end{align}
\qed
\section{Analysis of Main Algorithm \ref{algm:CI-Test} - Proof of Theorem \ref{thm:maintv}}

\textbf{Note:} To begin our analysis, please note that as in Proof of Lemma \ref{lem:restrict}, one can also differently have $\nu_{\pi}$ defined $\forall B\in {\cal F}$ as :
\begin{align}
\nu_{\pi} (B) = \mathbb{E}_{\pi} [ \mathbf{1}_{\{\mathbf{w}=\mathbf{\tilde{w}}= \mathbf{u}\}} \mathbf{1}_{\{\mathbf{u} \in B\}}\mathbb{E}_{\pi}[\mathbf{1}_{\{\mathbf{v}=\mathbf{\tilde{v}}\}}|\mathbf{w}=\mathbf{\tilde{w}}=\mathbf{u}, \mathbf{u}\in\ B]]\label{eqn:gdensity-1}
\end{align}
where $\pi\in \Pi^*$ is a coupling preserving $P$ and $Q$ as the marginal measures on tuple $(\mathbf{V}, \mathbf{W})$ which have densities $p(\cdot)$ and $q(\cdot)$ with respect to the  Lesbegue measure. Since $\mathbb{E}_{\pi}[\mathbf{1}_{\{\mathbf{v}=\mathbf{\tilde{v}}\}}|\mathbf{w}=\mathbf{\tilde{w}}=\mathbf{u}, \mathbf{u}\in\ B]\le 1$,  we have $\nu_{\pi}(B)\le P(\mathbf{W}\in B), Q(\mathbf{\tilde{W}}\in B)$. Thus by the same arguments as in the Lemma \ref{lem:restrict}, $\nu_{\pi}(\cdot)$ defined as per Eq. (\ref{eqn:gdensity-1}) also exhibits a density $g_{\pi}(\mathbf{u_w})$ with respect to the Lesbegue measure such that $g_{\pi}(\mathbf{u_w}) \leq p(\mathbf{u_w}),q(\mathbf{u_w})$ almost surely. 

 Now, the first equality in the theorem is obvious and follows from Algorithm \ref{algm:CI-Test}, Lemma \ref{lem:err}. We establish only the second inequality using the following chain:
  \begin{align}
    & D_{\mathrm{TV}}(p(\mathbf{z},\mathbf{x},\mathbf{y}),p(\mathbf{x},\mathbf{z}) q(\mathbf{y}|\mathbf{z}))  
      \nonumber \\
       \overset{(a)}{=} &1-  \max \limits_{\pi \in \Pi^{*}(p(\mathbf{z},\mathbf{x},\mathbf{y}),p(\mathbf{\tilde{z}}) q(\mathbf{\tilde{y}}|\mathbf{\tilde{z}}) p(\mathbf{\tilde{x}}|\mathbf{\tilde{z}}))} \mathbb{E}_{\pi}[\mathbf{1}_{\{\mathbf{x}=\mathbf{\tilde{x}},\mathbf{y}=\mathbf{\tilde{y}},\mathbf{z}=\mathbf{\tilde{z}}\}}] \nonumber \\ 
         \overset{}{=} &1-  \max \limits_{\pi \in \Pi^{*}(p(\mathbf{z},\mathbf{x},\mathbf{y}),p(\mathbf{\tilde{z}}) q(\mathbf{\tilde{y}}|\mathbf{\tilde{z}}) p(\mathbf{\tilde{x}}|\mathbf{\tilde{z}}))} \mathbb{E}_{\pi}[\mathbf{1}_{\{\mathbf{y}=\mathbf{\tilde{y}},\mathbf{z}=\mathbf{\tilde{z}}\}}\mathbb{E}_{\pi}[\mathbf{1}_{\{\mathbf{x}=\mathbf{\tilde{x}}\}}|\mathbf{y}=\mathbf{\tilde{y}},\mathbf{z}=\mathbf{\tilde{z}}]] \nonumber \\ 
     \overset{(b)}{=} & 1- \max \limits_{\pi \in \Pi^{*}(p(\mathbf{z},\mathbf{x},\mathbf{y}),p(\mathbf{\tilde{z}}) q(\mathbf{\tilde{y}}|\mathbf{\tilde{z}}) p(\mathbf{\tilde{x}}|\mathbf{\tilde{z}}))} \int g_{\pi}(\mathbf{u_y},\mathbf{u_z}) \mathbb{E}_{\pi}[\mathbf{1}_{\{\mathbf{x}=\mathbf{\tilde{x}}\}}|\mathbf{y}=\mathbf{\tilde{y}}=\mathbf{u_y},\mathbf{z}=\mathbf{\tilde{z}}=\mathbf{u_z}] d(\mathbf{u_y},\mathbf{u_z}) \nonumber \\
     \overset{(c)}{\geq} & 1- \max \limits_{\pi \in \Pi^{*}(p(\mathbf{z},\mathbf{x},\mathbf{y}),p(\mathbf{\tilde{z}}) q(\mathbf{\tilde{y}}|\mathbf{\tilde{z}}) p(\mathbf{\tilde{x}}|\mathbf{\tilde{z}}))}  \int g_{\pi}(\mathbf{u_y},\mathbf{u_z}) \epsilon(\mathbf{u_y},\mathbf{u_z}) d(\mathbf{u_y},\mathbf{u_z}) \nonumber \\
     \overset{(d)}{\geq} & 1 -\max \limits_{\pi \in \Pi(p(\mathbf{z},\mathbf{y}),p(\mathbf{\tilde{z}}) q(\mathbf{\tilde{y}}|\mathbf{\tilde{z}}))} \int g_{\pi}(\mathbf{u_y},\mathbf{u_z}) \epsilon(\mathbf{u_y},\mathbf{u_{z}}) d(\mathbf{u_y},\mathbf{u_z}) \nonumber \\
     \overset{(e)}{\geq} & 1 - \int \min (p(\mathbf{u_y},\mathbf{u_z}),p(\mathbf{u_z}) q(\mathbf{u_y}|\mathbf{u_z})) \epsilon(\mathbf{u_y},\mathbf{u_z}) d(\mathbf{u_y},\mathbf{u_z}) \label{eqn:firsttv}
  \end{align}
  We have the following justifications for the above inequality chain: 
  \begin{itemize}  
  \item[(a)] follows from the definition of the total variation distance as in Eq. (\ref{tv:cost}) and the restriction of coupling set $\Pi$ to $\Pi^*$ as observed via Lemma \ref{lem:restrict}.
  \item[(b)] follows from the note in the beginning of this section which alternatively defines $\nu_{\pi}(\cdot)$ as in Eq. (\ref{eqn:gdensity-1}). Thus $\nu_{\pi}(\Omega_{\pi)}= \mathbb{E}_{\pi}[\mathbf{1}_{\{\mathbf{y}=\mathbf{\tilde{y}},\mathbf{z}=\mathbf{\tilde{z}}\}}\mathbb{E}_{\pi}[\mathbf{1}_{\{\mathbf{x}=\mathbf{\tilde{x}}\}}|\mathbf{y}=\mathbf{\tilde{y}},\mathbf{z}=\mathbf{\tilde{z}}]]$
   \item[(c)] follows from the fact that given $\mathbf{y}=\mathbf{\tilde{y}}=\mathbf{u_y}$ (fixed constant) and similarly $\mathbf{z}=\mathbf{\tilde{z}}=\mathbf{u_z}$, $\mathbb{E}_{\pi}[\mathbf{1}_{\{\mathbf{x}=\mathbf{\tilde{x}}\}}|\mathbf{u_y}, \mathbf{u_z}]=P_{\pi}(\mathbf{x}=\mathbf{\tilde{x}}|\mathbf{u_y},\mathbf{u_z})$ is the probability with respect to some coupling $\pi$ between $p(\mathbf{x}|\mathbf{u_z})$ and $p(\mathbf{\tilde{x}}|\tilde{u_y},\mathbf{u_z})$. Any such probability is bounded above by $\epsilon(\mathbf{u_y},\mathbf{u}_z)$ according to equation (\ref{eqn:coupling}).
  \item[(d)] follows from the fact that terms inside the integral in (c) depend only on the marginal coupling under $\pi$ between $(\mathbf{y},\mathbf{z})$ variables.
  \item[(e)] is a consequence of the note in the beginning of this section.
   \end{itemize} 
  
  We also have the following equality as a result of Corollary \ref{corollary0}:
   \begin{align}\label{eqn:secondtv}
      D_{\mathrm{TV}}(p(\mathbf{z}) q(\mathbf{y}|\mathbf{z}), p(\mathbf{y},\mathbf{z})) = 1 - \int \min (p(\mathbf{u_y},\mathbf{u_z}),p(\mathbf{u_z}) q(\mathbf{u_y}|\mathbf{u_z})) d(\mathbf{u_y},\mathbf{u_z}) 
   \end{align}
 
 Subtracting (\ref{eqn:secondtv}) from (\ref{eqn:firsttv}), we have:
   \begin{align}\label{eqn:finaltv}
      & D_{\mathrm{TV}}(p(\mathbf{z},\mathbf{x},\mathbf{y}),p(\mathbf{x},\mathbf{z}) q(\mathbf{y}|\mathbf{z})) - D_{\mathrm{TV}}(p(\mathbf{z}) q(\mathbf{y}|\mathbf{z}), p(\mathbf{y},\mathbf{z})) \\
      & \geq \int \min (p(\mathbf{u_y},\mathbf{u_z}),p(\mathbf{u_z}) q(\mathbf{u_y}|\mathbf{u_z})) (1-\epsilon(\mathbf{u_y},u_{\mathbf{z}})) d(\mathbf{u_y},\mathbf{u_z})
   \end{align}
\qed

\section{Additional Theorems in proving Theorem \ref{thm:mainthm1}}

\subsection{Proof of Theorem \ref{theorem3}}
 The result follows from Theorem \ref{thm:maintv} and the fact that conditional dependence implies that $\epsilon(\textbf{y},\mathbf{z})< 1$ is true with a non zero-measure with respect to distribution given by the density $p(\textbf{y},\mathbf{z})$. \qed

\subsection{Proof of Theorem \ref{theorem4}}

  Conditional independence means that $p(\mathbf{x},\mathbf{y},\mathbf{z})=p(\mathbf{z})p(\mathbf{y}|\mathbf{z})p(\mathbf{x}|\mathbf{z})$. Further, 
  \begin{align}\label{eq:twotvs}
  D_{\mathrm{TV}} (p(\mathbf{z})p(\mathbf{y}|\mathbf{z}),p(\mathbf{z})q(\mathbf{y}|\mathbf{z})) = D_{\mathrm{TV}} (p(\mathbf{x}|\mathbf{z})p(\mathbf{z})p(\mathbf{y}|\mathbf{z}),p(\mathbf{x}|\mathbf{z})p(\mathbf{z})q(\mathbf{y}|\mathbf{z})) 
  \end{align}
Combining (\ref{eq:twotvs}) with the first equality in Theorem \ref{thm:maintv}, we have the result stated.  \qed

\section{Additional Corollaries of Theorem \ref{thm:mainthm1}}
\subsection{Proof of Corollary \ref{corollary1}}

We first observe that due to triangle inequality we have the following:
 \begin{align}\label{eqn:triang}
 & D_{\mathrm{TV}}(p(\mathbf{z},\mathbf{x},\mathbf{y}),p(\mathbf{z}) q(\mathbf{y}|\mathbf{z}) p(\mathbf{x}|\mathbf{z})) - D_{\mathrm{TV}}(p(\mathbf{z}) q(\mathbf{y}|\mathbf{z}), p(\mathbf{y},\mathbf{z})) \nonumber \\
  & \overset{(a)}{=} D_{\mathrm{TV}}(p(\mathbf{z},\mathbf{x},\mathbf{y}),p(\mathbf{z}) q(\mathbf{y}|\mathbf{z}) p(\mathbf{x}|\mathbf{z})) - D_{\mathrm{TV}}(p(\mathbf{x}|\mathbf{z})p(\mathbf{z}) q(\mathbf{y}|\mathbf{z}), p(\mathbf{x}|\mathbf{z}) p(\mathbf{y},\mathbf{z})) \nonumber \\
  & \overset{(b)}{\leq} D_{\mathrm{TV}}(p(\mathbf{z},\mathbf{x},\mathbf{y}),p(\mathbf{z}) p(\mathbf{y}|\mathbf{z}) p(\mathbf{x}|\mathbf{z}))
 \end{align}
 where (a) follows from Eq. (\ref{eq:twotvs}) and (b)  follows from triangle inequality for the total variation distance.
 
When you set $q(\mathbf{y}|\mathbf{z})=p(\mathbf{y}|\mathbf{z})$, lower bound in Theorem \ref{thm:maintv} matches the upper bound in Eq. (\ref{eqn:triang}). This implies the variational result stated. \qed

\subsection{Proof of Corollary \ref{corollary2}}

   Given the conditions, we observe that $q(\cdot) \geq p(\cdot) \frac{1}{2ab}$. Hence, the statement of Theorem \ref{thm:maintv} can be written as: 
   \begin{align}
      2\lvert \mathbb{E}_D[e(f^{*}_1,D_s)]- \mathbb{E}_D[e(f^{*}_2,D^{-\mathbf{x}}_s)]  \rvert & \geq 
      \frac{1}{2ab}\int \limits_{y,\mathbf{z}} p(\mathbf{z}) p(y|\mathbf{z}) (1-\epsilon(y,\mathbf{z})) d(y,\mathbf{z}) \nonumber \\
      \hfill & \overset{(a)}{=} \frac{1}{2ab} \int \limits_{y,\mathbf{z}} p(\mathbf{z}) p(y|\mathbf{z})  \left( \inf_{\pi \in \Pi(p(\mathbf{x}|\mathbf{z}),p(\mathbf{x}'|y,\mathbf{z}))} \mathbb{E}_{\pi}[\mathbf{1}_{\{\mathbf{x} \neq \mathbf{x}'\}}] \right) d(y,\mathbf{z}) \nonumber \\
      \hfill & \overset{(b)}{=}  \frac{1}{2ab} \int \limits_{y,\mathbf{z}} p(\mathbf{z}) p(y|\mathbf{z}) D_{\mathrm{TV}}(p(\mathbf{x}|\mathbf{z}),p(\mathbf{x}|y,\mathbf{z})) d(y,\mathbf{z}) \nonumber \\
      \hfill & = \frac{1}{2ab}  D_{\mathrm{TV}} (p(y,\mathbf{z})p(\mathbf{x}|\mathbf{z}),p(y,\mathbf{z}) p(\mathbf{x}|y,\mathbf{z}))
    \end{align}   
where (a) is by the definition of $\epsilon(y,\mathbf{z})$ and (b)  is due to the definition of total variation distance as in Eq. (\ref{tv:cost}).\qed

\section{General Measures - Proof of Theorem \ref{thm:mainthm2}}

 Let $P(\cdot)$ and $Q(\cdot)$ be two different probability measures on $\mathbf{R}^{1 \times s}$ and absolutely continuous with another measure $\mu$. Hence corresponding Radon Nikodym derivatives are respectively, $\frac{dP}{d\mu}$ and $\frac{dQ}{d\mu}$. Form a data set $D$ containing $p$ i.i.d. samples of the form $(\mathbf{w},\ell)$ where $\ell$ is a Bernoulli random variable with bias probability $0.5$. If $\ell=1$, then $\mathbf{w} \sim P(\cdot)$ and when $\ell=0$, then $\mathbf{w} \sim Q(\cdot)$. Considering the space of classifiers $f:\mathbf{w} \rightarrow \{0,1\}$ and $\mathbb{E}_{D}[e(f,D)]$, the expected error for any classifier $f(\cdot)$, we have:
 
 \begin{lemma}
 The Bayes optimal classifier denoted by $f^{*}:\mathbf{w} \rightarrow \{0,1\}$ is: $\ell=1$ if $\mathbf{w}\in A$ and $\ell=0$ otherwise, where $A = \{\mathbf{w} : \frac{dP}{d\mu}(\mathbf{w}) > \frac{dQ}{d\mu} (\mathbf{w})\}$ and
 
 \begin{align}
     & \min \limits_{f} \mathbb{E}_D[e(f,D)] \nonumber \\
     & =  \mathbb{E}_D[e(f^{*},D)] \nonumber \\
     & =  \frac{1}{2}\int \min (\frac{dP}{d\mu}(\mathbf{w}),\frac{dQ}{d\mu}(\mathbf{w})) d\mu(\mathbf{w}) \nonumber\\
     & =  \frac{1}{2}-\frac{1}{2} D_{\mathrm{TV}}(P,Q) \label{chain-dc-1}
 \end{align}
 
Further we can have $\pi$ restricted to $\Pi^*$ and $\nu_{\pi}(\cdot)$ can be defined equivalently here as in Proof for Lemma \ref{lem:restrict}. We then have $\nu_{\pi}$ absolutely continuous with respect to $\mu$ such that $\frac{d\nu_{\pi}}{d\mu}\leq \min (\frac{dP}{d\mu},\frac{dQ}{d\mu})$. 
 \end{lemma}
 
\begin{proof}
This follows the same line of treatment as in corresponding lemmas in Section \ref{sec:bayesoptimal} by noting that now Bayes Optimal classifier is as specified in terms of Radom Nikodym derivatives and that the total variation distance is given by:
   \begin{align}
      D_{\mathrm{TV}}(P,Q) & =  \sup_{A \in {\cal F}} \lvert P(A) - Q(A) \rvert = \frac{1}{2} \int \lvert \frac{dP}{d\mu}(\mathbf{w}) -\frac{dQ}{d\mu}(\mathbf{w}) \rvert d\mu(\mathbf{w}) 
   \end{align}

\end{proof}

   
Now defining $\epsilon(\mathbf{y}, \mathbf{z})$ in terms of conditional measures, we have analogous result for Theorem \ref{thm:maintv} in case of general measures for probability distribution  $\mathbb{P}$ and $\mathbb{Q}$, i.e.:
  \begin{align}
  & 2\lvert \mathbb{E}_D[e(f^{*}_1,D_s)]- \mathbb{E}_D[e(f^{*}_2,D^{-\mathbf{x}}_s)]  \rvert \nonumber\\
  &= \lvert D_{\mathrm{TV}}(\mathbb{P},\mathbb{Q}) - D_{\mathrm{TV}}(\mathbb{P}_{YZ},\mathbb{Q}_{YZ}) \rvert \nonumber\\
  \hfill & \geq \int \limits_{\mathbf{y},\mathbf{z}} \lvert \min (\frac{d\mathbb{P}_{YZ}}{d\mu}(\mathbf{y},\mathbf{z}) , \frac{d\mathbb{Q}_{YZ}}{d\mu}(\mathbf{y},\mathbf{z}) ) \rvert (1-\epsilon(\mathbf{y},\mathbf{z})) d\mu(\mathbf{y},\mathbf{z}) \nonumber
  \end{align}
where we have assumed that $\exists \mu$ such that $\mathbb{P}$ and $\mathbb{Q}$ are absolutely continuous with respect to $\mu$. We omit the proof as is essentially based on similar algebra as in the proof of Theorem \ref{thm:maintv}. If $\mathbb{Q}_{YZ}$ is absolutely continuous with respect to $\mathbb{P}_{YZ}$, then we have:
  \begin{align}
  2\lvert \mathbb{E}_D[e(f^{*}_1,D_s)]- \mathbb{E}_D[e(f^{*}_2,D^{-\mathbf{x}}_s)]  \rvert \geq \int \limits_{\mathbf{y},\mathbf{z}} \lvert \min (\frac{d\mathbb{Q}_{YZ}}{d\mathbb{P}_{YZ}}(\mathbf{y},\mathbf{z}),1 ) \rvert (1-\epsilon(\mathbf{y},\mathbf{z})) d\mathbb{P}_{YZ}(\mathbf{y},\mathbf{z}) \nonumber
  \end{align}
Corresponding to Theorem \ref{thm:mainthm1} (by mimicking the Theorem \ref{theorem3} and Theorem \ref{theorem4}), we have the Theorem \ref{thm:mainthm2} as proposed. \qed

\end{appendices}

\end{document}